\documentclass[a4paper]{article}

\usepackage[T1]{fontenc}
\usepackage[round,authoryear]{natbib}

\usepackage[colorlinks=true,linkcolor=blue, citecolor=blue, urlcolor=blue]{hyperref}
\usepackage{caption}
\usepackage{graphicx}%
\usepackage{multirow}%
\usepackage{amsmath,amssymb,amsfonts}%
\usepackage{amsthm}%
\usepackage{mathrsfs}%
\usepackage[title]{appendix}%
\usepackage{xcolor}%
\usepackage{textcomp}%
\usepackage{manyfoot}%
\usepackage{booktabs}%
\usepackage{algorithm}%
\usepackage{algorithmicx}%
\usepackage[noend]{algpseudocode}%
\usepackage{listings}%

\usepackage[caption=false,font=footnotesize]{subfig}
\usepackage{enumerate}
\usepackage{enumitem}
\usepackage{multirow}
\usepackage{mathtools}
\usepackage{tikz}
\usepackage{xparse}
\usepackage{lastpage}

\usepackage{authblk}

\newtheorem{theorem}{Theorem}

\newtheorem{lemma}[theorem]{Lemma}
\newtheorem{proposition}[theorem]{Proposition}
\newtheorem{problem}[theorem]{Problem}
\newtheorem{remark}[theorem]{Remark}
\newtheorem{definition}[theorem]{Definition}

\newcommand{\keywords}[1]{\par\addvspace\baselineskip
\noindent\keywordname\enspace\ignorespaces\begingroup\def\and{\unskip\ \raisebox{0.35ex}{\scalebox{0.5}{\textbullet}}\ }%
#1%
\endgroup}
\newcommand{\keywordname}{\textbf{Keywords:}}

\algnewcommand\algorithmicinput{\textbf{Input:}}
\algnewcommand\Input{\item[\algorithmicinput]}
\algnewcommand\algorithmicoutput{\textbf{Output:}}
\algnewcommand\Output{\item[\algorithmicoutput]}

\definecolor{tabblue}{RGB}{31,119,180}
\definecolor{taborange}{RGB}{255,127,14}
\definecolor{tabgreen}{RGB}{44,160,44}
\definecolor{tabred}{RGB}{214,39,40}
\definecolor{tabpurple}{RGB}{148,103,189}
\definecolor{tabbrown}{RGB}{140,86,75}
\definecolor{tabpink}{RGB}{227,119,194}
\definecolor{tabgray}{RGB}{127,127,127}
\definecolor{tabolive}{RGB}{188,189,34}
\definecolor{tabcyan}{RGB}{23,190,207}

\definecolor{fuchsia}{HTML}{FF00FF}

\NewDocumentCommand{\tagbox}{O{red} m}{\tikz[baseline]{\node[anchor=base,text=#1,fill=#1!25,font=\sffamily,rounded corners=1mm,text depth=0.5mm] {\upshape\vphantom{hg}#2};}\,}

\newcommand{\N}{\mathbb{N}}
\newcommand{\R}{\mathbb{R}}

\NewDocumentCommand{\euler}{}{\mathrm{e}}

\DeclarePairedDelimiter{\abs}{\lvert}{\rvert} 

\DeclarePairedDelimiter{\card}{\lvert}{\rvert}

\DeclareMathOperator{\vol}{vol}

\DeclareMathOperator{\cl}{cl}

\NewDocumentCommand{\power}{m}{2^{#1}}
\NewDocumentCommand{\finitesubsets}{O{X}}{[#1]^{<\infty}}

\newcommand{\bigo}{\mathcal{O}}

\DeclareMathOperator{\ext}{ext} 

\NewDocumentCommand{\connected}{D<>{\theta} O{C}}{\sim_{#1,#2}}
\NewDocumentCommand{\nconnected}{D<>{\theta} O{C}}{\nsim_{#1,#2}}
\NewDocumentCommand{\conncomp}{D<>{\theta} O{C} m}{[#3]_{\connected<#1>[#2]}}

\NewDocumentCommand{\witness}{D<>{\theta}}{\mathcal{I}_{#1}}

\NewDocumentCommand{\vcdim}{}{\operatorname{VC}_{\mathrm{dim}}}

\newcommand{\clop}{\convhullsymbol}
\newcommand{\closuresys}{\mathcal{C}}

\newcommand{\convhullsymbol}{\rho}

\NewDocumentCommand{\conv}{}{\convhullsymbol}
\NewDocumentCommand{\ConvSets}{o}{\WConvSets[#1]<\infty>}

\newcommand{\ti}{\mathcal{I}}  

\newcommand{\wconvhullsymbol}{\convhullsymbol}
\NewDocumentCommand{\wconv}{D<>{\theta}}{\wconvhullsymbol_{#1}}
\NewDocumentCommand{\WConvSets}{o D<>{\theta}}{\mathcal{C}_{#2\IfNoValueF{#1}{,#1}}}

\NewDocumentCommand{\prewconv}{D<>{\theta}}{\hat{\wconvhullsymbol}_{#1}}
\NewDocumentCommand{\appwconv}{D<>{\theta}}{\tilde{\wconvhullsymbol}_{#1}}

\NewDocumentCommand{\repwconv}{D<>{\theta} o}{\mu\IfNoValueF{#2}{(#1,#2)}}
\NewDocumentCommand{\blockrepwconv}{D<>{\theta} o}{\mu'\IfNoValueF{#2}{(#1,#2)}}

\NewDocumentCommand{\blocks}{d<>}{\mathcal{B}\IfNoValueF{#1}{_{#1}}}

\NewDocumentCommand{\fn}{m}{\textproc{#1}}

\NewDocumentCommand{\textsymbol}{m}{\text{\sc #1}}

\NewDocumentCommand{\symbolno}{}{\textsymbol{No}}
\NewDocumentCommand{\symboltrue}{}{\textsymbol{True}}
\NewDocumentCommand{\symbolfalse}{}{\textsymbol{False}}

\newcommand{\M}{\mathcal{M}}
\newcommand{\dist}{D} 

\NewDocumentCommand{\neighborhood}{O{r} m}{N_{#1}(#2)}

\NewDocumentCommand{\conceptclass}{m}{\mathcal{#1}}
\NewDocumentCommand{\kfoldunion}{O{k} m}{{#2}_{\cup}^{#1}}

\NewDocumentCommand{\gridGraph}{O{\ell}}{G_{#1}}
\NewDocumentCommand{\periodicGridGraph}{O{\ell}}{\gridGraph[#1]^{\circlearrowleft}}

\newcommand{\MHamming}{\M_{\Hamming}}
\newcommand{\Hamming}{H}
\newcommand{\dHamming}{\dist_{H}}
\NewDocumentCommand{\blockrepwconvHamming}{}{\blockrepwconv_{\Hamming}}
\NewDocumentCommand{\repwconvHamming}{}{\repwconv_{\Hamming}}

\NewDocumentCommand{\ccWConvBoolean}{D<>{\theta} O{n}}{\conceptclass{F}_{#2,#1}}

\NewDocumentCommand{\Rd}{O{d}}{\R^{#1}}
\NewDocumentCommand{\minkowskidist}{m}{\dist_{#1}}
\NewDocumentCommand{\euclideandist}{}{\minkowskidist{2}}

\NewDocumentCommand{\dball}{O{d} D<>{r} O{1} m}{\overline{\mathrm{B}}_{#1, #2}^{#3}(#4)}

\NewDocumentCommand{\cball}{D<>{r}}{\overline{\mathrm{B}}_{#1}}
\NewDocumentCommand{\oball}{D<>{r}}{\mathrm{B}_{#1}}

\NewDocumentCommand{\MPlane}{}{\M_{\mathrm{2}}}
\NewDocumentCommand{\repwconvPlane}{}{\repwconv_{\mathrm{2}}}
\NewDocumentCommand{\blockrepwconvPlane}{}{\blockrepwconv_{\mathrm{2}}}
\NewDocumentCommand{\extreme}{}{\operatorname{extr}}
\NewDocumentCommand{\ccextreme}{}{\extreme^{\circlearrowleft}}

\newcommand{\unitdcubesymbol}{U}
\NewDocumentCommand{\MUnitCube}{}{\M_{\unitdcubesymbol}}
\NewDocumentCommand{\unitdcube}{O{d}}{\unitdcubesymbol_{#1}}

\NewDocumentCommand{\blockrepwconvUnitCube}{}{\blockrepwconv_{\unitdcubesymbol}}
\NewDocumentCommand{\repwconvUnitCube}{}{\repwconv_{\unitdcubesymbol}}

\NewDocumentCommand{\ccWConvRectangles}{D<>{\theta} O{d}}{\conceptclass{HR}_{#2,#1}}

\newcommand{\TS}{T_S}
\newcommand{\TD}{T_D}
\newcommand{\TJ}{T_J}
\newcommand{\TM}{T_M}

\newcommand{\decomp}{\mathcal{R}}

\newcommand{\cC}{\mathcal{C}}

\begin{document}

\title{Learning weakly convex sets in metric spaces}

\date{} 

\author[1,3]{Eike Stadtl\"ander\thanks{Email: \texttt{stadtlaender@cs.uni-bonn.de}}}
\author[1,2,3]{Tam\'as Horv\'ath\thanks{Email: \texttt{horvath@cs.uni-bonn.de}}}
\author[1,2,3]{Stefan Wrobel\thanks{Email: \texttt{wrobel@cs.uni-bonn.de}}}

\affil[1]{Dept. of Computer Science, University of Bonn, Germany}
\affil[2]{Fraunhofer IAIS, Schloss Birlinghoven, Sankt Augustin, Germany}
\affil[3]{Lamarr Institute for Machine Learning and Artificial Intelligence} 

\maketitle

\begin{abstract}
    One of the central problems studied in the theory of machine learning is the question of whether, for a given class of hypotheses, it is possible to efficiently find a {consistent} hypothesis, i.e., which has zero training error. 
    While problems involving {\em convex} hypotheses have been extensively studied, the question of whether efficient learning is possible for {\em non-convex} hypotheses composed of possibly several disconnected regions is still less understood.
    Although it has been shown quite a while ago that efficient learning of weakly convex hypotheses, a parameterized relaxation of convex hypotheses, is possible for the special case of Boolean functions, the question of whether this idea can be developed into a \textit{generic paradigm} has not been studied yet.
    In this paper, we provide a positive answer and show that the consistent hypothesis finding problem can indeed be solved in polynomial time for a broad class of weakly convex hypotheses over metric spaces. 
    To this end, we propose a general domain-independent algorithm for finding consistent weakly convex hypotheses and prove sufficient conditions for its efficiency that characterize the corresponding hypothesis classes.
    To illustrate our general algorithm and its properties, we discuss several non-trivial learning examples to demonstrate how it can be used to efficiently solve the corresponding consistent hypothesis finding problem.
    Without the weak convexity constraint, these problems are 
    known to be computationally intractable.
    We then proceed to show that the general idea of our algorithm can even be extended to the case of {\em extensional} weakly convex hypotheses, as it naturally arise, e.g., when performing vertex classification in graphs. 
    We prove that using our extended algorithm, the problem can be solved in polynomial time provided the distances in the domain can be computed efficiently.

    \keywords{concept learning \and consistent hypothesis finding \and intersection-closed concept classes \and convexity \and closure systems}
\end{abstract}

\section{Introduction}
\label{sec:introduction}
One of the central problems of concept learning is the \textit{consistent hypothesis finding} (CHF) problem defined as follows: Given a set of positive and negative examples of an unknown target concept, find a consistent hypothesis, i.e., which has zero training error, from the underlying hypothesis class if such a hypothesis exists. 
This problem has been extensively studied for \textit{convex} hypothesis classes defined by geodesic convexity over metric spaces~\citep{Menger1928}.
Examples include axis-aligned hyperrectangles \citep{Blumer1989}, shortest-path convexity in graphs \citep{Seiffarth2023}, or conjunctions \citep{Valiant1984}. 
The CHF algorithms typically utilize that convex hypothesis classes
can be characterized by a \textit{convex hull} closure operator.
A major weakness of convex hypotheses is that they are composed of a \textit{single} contiguous block, which severely limits the expressive power of convex hypothesis classes.
This limitation motivates the question of whether efficient learning is possible for {\em non-convex} hypotheses, i.e., those consisting of possibly several disconnected, not necessarily convex regions. 

\cite{Ekin2000} give an affirmative answer to this question for the \textit{special} case of Boolean functions by showing that the CHF problem can be solved efficiently for the class of functions satisfying $k$-convexity, a relaxed notion of convexity defined as follows: An $n$-ary Boolean function $f$ is \textit{$k$-convex} for some $k \geq 0$ if all points on all shortest paths between two true points of $f$ with Hamming distance at most $k$ are also true points of $f$.
The true points of such a \textit{weakly convex} Boolean function form a set of subcubes of the $n$-dimensional Boolean hypercube that have a pairwise distance greater than $k$. 
Accordingly, a $k$-convex Boolean function can be represented by a disjunctive normal form (DNF).
It is a well-known result that, under widely believed complexity assumptions, finding a consistent DNF with the smallest number of terms, i.e., subcubes of the $n$-dimensional Boolean hypercube, is NP-hard, if there is no constraint on the relationship between the subcubes~\citep{Pitt1988}. 
Somewhat surprisingly, to the best of our knowledge, the question of whether the idea of $k$-convexity can be developed into a generic paradigm for solving the CHF problem for convex hypothesis classes over \textit{other} metric spaces has not been studied yet.

In this work, we close this gap and provide a positive answer by adapting the above idea in \citep{Ekin2000}. 
We show that the CHF problem can indeed be solved efficiently for a broad class of \textit{weakly convex} hypotheses over metric spaces, where a subset $A$ of a metric space is weakly convex if it is (topologically) closed and for all $x, y \in A$ and $z$ in the domain, $z$ belongs to $A$ whenever $x$ and $y$ are \textit{close} to each other with respect to a threshold $\theta$ and the three points satisfy the triangle inequality with equality.
To this end, we first prove that weakly convex sets give rise to a \textit{unique} decomposition into a set of ``connected'' blocks that have a pairwise distance greater than $\theta$ and that the weakly convex hulls of a set grow monotonically  with $\theta$, while  their number of  ``contiguous'' blocks decreases. 
We also show that weakly convex hypothesis classes are intersection closed. 
Using these results, we provide a general \textit{domain-independent} algorithm for solving the CHF problem for weakly convex hypothesis classes and prove sufficient conditions for the efficiency of this algorithm.
Our CHF algorithm assumes that the hypotheses are given \textit{intensionally}, i.e., by some property.
Its solution is \textit{optimal} in the sense that it computes the consistent weakly convex hull of the positive examples that has the \textit{smallest} number of blocks.

To illustrate our general algorithm and its properties, we consider the CHF problem for unions of Boolean hypercubes, axis-aligned hyperrectangles, and convex polygons. While, for example, finding a consistent hypothesis with the smallest number of axis-aligned hyperrectangles is NP-hard in general~\citep{Bereg2012}, for weakly convex unions of axis-aligned hyperrectangles we show in a fairly simple way how our algorithm can be used to efficiently solve the CHF problem. Using this result, we then prove that weakly convex unions of axis-aligned hyperrectangles are polynomially PAC-learnable.

We also show that the general idea behind our algorithm can even be extended to the case that weakly convex hypotheses are given {\em extensionally}, i.e., by enumerating their elements, as it naturally arises, for example, in vertex classification in graphs. 
For this setting we show that our extended algorithm computes the consistent weakly convex hull of the positive examples with the \textit{smallest} number of blocks in \textit{polynomial} time, if the distance matrix for the domain can be computed efficiently.

\noindent
\textbf{Outline} \ The rest of the paper is organized as follows.
We overview the related work in Section~\ref{sec:relatedwork}
and collect the necessary notions 
in Section~\ref{sec:preliminaries}.
In Section~\ref{sec:weakconvexity}, we define weakly convex sets 
and prove some of their basic properties.
Sections~\ref{sec:intensional-setting} and \ref{sec:extensional-setting} are devoted to the algorithm learning weakly convex hypotheses in the intensional problem setting and to its extended version for the extensional case, respectively.
Finally, we conclude in Section~\ref{sec:conclusion} and mention some problems for future work.

\noindent
\textbf{Remark} \ \textit{A short version of this paper appeared in \citep{Stadtlaender2021}.}

\section{Related Work} 
\label{sec:relatedwork}
The CHF problem and its variants have been intensively studied also in other fields of computer science.
For example, a closely related problem considered in discrete algorithms is the \textit{red-blue set covering problem}~\citep{Carr2000} defined as follows: 
Given disjoint finite sets $R$ and $B$ of red and blue points and the trace\footnote{The trace $\mathcal{F}_{|Y}$ of a set system $\mathcal{F} \subseteq 2^X$ on a set $Y$ is the set system $\{S \cap Y: S \in \mathcal{F}\}$.}
$\mathcal{S}_{|R\cup B}$ of a set system $\mathcal{S}$, find a family $\mathcal{S}' \subseteq \mathcal{S}$ such that $\mathcal{S}'$ covers all blue points and as few red points as possible.
Thus, in contrast to the CHF problem, the goal is to \textit{minimize} the number of red points covered by $\mathcal{S}'$ (i.e., which are ``misclassified''), and not $|\mathcal{S}'|$.
This problem is NP-hard even for the cases that $R, B \subseteq \R^2$ and $\mathcal{S}$ is the family of axis-aligned unit squares~\citep{Chan2015} or that of axis-aligned rectangles~\citep{Abidha2024}.
As another example, we mention computational geometry, where the CHF problem is called the \textit{class cover problem} and defined as follows: 
Given disjoint finite sets $R$ and $B$ of red and blue points and a set system $\mathcal{S} \subseteq \power{R \cup B}$, find a family $\mathcal{S}' \subseteq \mathcal{S}$ of the \textit{smallest} cardinality that covers $B$ and is disjoint with $R$ (i.e., no missclassification is allowed). 
Various special cases of this problem have been studied in this research field.
For example, this problem is NP-complete if $\mathcal{S}$ is the trace of a set of balls centered at the blue points~\citep{Cannon2004} or that of all axis-aligned rectangles of the plane~\citep{Bereg2012} on $R \cup B$.
We also mention the \textit{red-blue line separation} problem defined as follows: Given disjoint finite sets $R, B \subseteq \R^2$ of red and blue points and a positive integer $k$, decide if there is a set of at most $k$ lines that separate the red and blue points. 
This problem is NP-complete~\citep{Megiddo1988} and remains NP-complete even for the case that the lines are required to be axis-parallel~\citep{Calinescu2005}. 
The separating lines in this case define rectangular areas and the task is to cover all points with monochromatic axis-aligned rectangles.

Our paper deals with the CHF problem for \textit{intersection-closed} hypothesis classes.
There are several results for this special case in machine learning~\citep[see, e.g.,][]{Auer1998,Blumer1989,Helmbold1990,Natarajan1987,Pitt1988}.
Such hypothesis classes are \textit{closure systems} and their elements can be characterized by a \emph{closure operator}~\citep[see, e.g.,][]{Davey2002}. 
We consider closure systems over \textit{arbitrary} metric spaces defined by \textit{geodesic convexity}~\citep{Menger1928,vandeVel1993}.
Recently, there has been an increasing interest in learning this kind of closure systems over \textit{graphs}. 
Examples include vertex classification
\citep{deAraujo2019,Seiffarth2023,Thiessen2021,Thiessen2022} and recovering clusterings \citep{Bressan2021}.
Unless otherwise specified, by convex sets we always mean geodesically convex sets over metric spaces.

Convex sets form \textit{single} regions that are ``connected'', i.e., contain all elements that lie ``between'' any two of their elements. However, this property can make them
too restrictive for learning scenarios, raising the question whether efficient learning is possible for \textit{non-convex} hypotheses that consist of possibly several well-separated, not necessarily convex regions.
A natural step in this direction could be to consider the 
\textit{generalized} CHF problem: Find $k$ hypotheses for a given (or equivalently, for the smallest) $k$ such that their union is consistent with the examples. 
However, this problem can be computationally intractable; examples include the infeasibility of deciding the existence of a consistent $k$-term-DNF for any $k \geq 2$~\citep{Pitt1988}.

In the above approach, there is no restriction on convex sets.
In contrast, our general purpose algorithm requires a minimum distance between the regions to guarantee efficiency. 
It is inspired by the definition of \textit{$k$-convex} Boolean functions~\citep{Ekin2000}, where the convexity condition must only hold for such points of the Hamming space that have a distance at most a threshold $k$.
\cite{Ekin2000} show that the CHF problem can be solved \textit{efficiently} for $k$-convex Boolean functions, which are strict extensions of single conjunctions.
The same relaxation of convexity to $k$-convexity or to very similar notions have also been studied for various types of \textit{discrete} metric spaces and for \textit{fixed} values of $k$.\footnote{Victor Chepoi, private communication, 2021.}
Examples include weakly modular graphs \citep{Chepoi1989}, $\Delta$-matroids and basis graphs of matroids \citep{Chepoi2007}, ample classes \citep{Chalopin2022}, and hypercellular graphs \citep{Chepoi2020}.
Similar to our work, these papers are all concerned with some ``local to global convexity'' results.
However, they do \textit{not} discuss their algorithmic and learnability aspects.

In the context of unsupervised learning, \citet{Bressan2021} introduce a distance-based notion of convexity for graphs, more precisely, for clusterings of the vertices of a graph.
It differs, however, from our definition of weak convexity applied to graphs in at least two aspects: 
First, in a broad sense, they control the inter- and intra-cluster distances with \textit{two} parameters instead of a single one. 
Second, their notion of ``convex hull'' is induced by a finite set of simple paths of bounded length that depends on the geodesic distance between their endpoints.
In contrast, our notion of weak convexity is based on the set of all shortest paths of length bounded by a \textit{static} threshold.

\section{Preliminaries} 
\label{sec:preliminaries} 
In this section, we collect the necessary notions and fix the notation.
For any $n \in \N$ and $\tau \in \R$, $[n]$ and $\R_{\geq \tau}$ denote the sets $\{1,2,\dots,n\}$ and $\{x \in \R: x \geq \tau\}$, respectively. 
The family of all finite subsets of a set $X$ is denoted by $\finitesubsets[X]$.
A \emph{metric space} $\M$ is a pair $(X, \dist)$, where $X$ is a set and $\dist: X \times X \to \R_{\geq 0}$ is a metric on $X$.
$\M$ is complete if every Cauchy sequence in $\M$ converges to an element of $\M$. 
It follows that finite metric spaces are complete.
A subset $A \subseteq X$ is \textit{closed} if it contains all of its limit points.
For a subset $A$ of a metric space $\M$, $\cl(A)$ denotes the smallest closed subset of $\M$ that contains $A$.
The distance between two sets $A, B \subseteq X$ is defined by $D(A, B) = \inf \{D(a, b) : a \in A, b \in B\}$. 
The Manhattan and the Euclidean distances in $\Rd$ are denoted by $D_1$ and $D_2$, respectively.

A \emph{closure system} over some ground set $X$ is a pair $(X,\closuresys)$ with $\closuresys \subseteq \power{X}$ such that $\closuresys$ is closed under arbitrary intersection, where $\power{X}$ denotes the power set of $X$.
We assume that $X \in \closuresys$. 
One of the elementary properties of closure systems is that they can be characterized in terms of fixed points of closure operators \citep[see, e.g.,][]{Davey2002}. 
More precisely, a function $\rho: \power{X} \to \power{X}$ is a \emph{closure operator} if it satisfies the following properties for all $A, B \subseteq X$: (i) $A \subseteq \rho(A)$ (extensivity), (ii) $\rho(A) \subseteq \rho(B)$ if $A \subseteq B$ (monotonicity), and (iii) $\rho(\rho(A)) = \rho(A)$ (idempotency). 
If $\rho$ is extensive and monotone, but not necessarily idempotent, then it is a \emph{preclosure operator}. 
The fixed points of a closure operator $\rho$ are called \textit{$\rho$-closed} and the set system $(X,\closuresys_\rho)$ with $\closuresys_\rho = \{A \subseteq X : \rho(A) = A\}$ is always a closure system. 
Conversely, for any closure system $(X,\closuresys)$, the function $\rho: \power{X} \to \power{X}$ with $\rho(A) = \bigcap\{C \in \closuresys: A \subseteq C\}$ for all $A \subseteq X$ is a closure operator satisfying $\closuresys = \{\rho(A) : A \subseteq X\}$. 
One can easily check that the set operator $\cl$ defined above is a closure operator.

Our notion of \emph{weak convexity} is inspired by that of \emph{$k$-convexity} introduced by \citet{Ekin1999}. 
More precisely, for the metric space $\MHamming=(\Hamming_n,\dHamming)$, called the Hamming metric space, where $\Hamming_n = \{0,1\}^n$ is the $n$-dimensional \textit{Hamming} or \textit{Boolean cube} and $\dHamming$ is the $L_1$  or Hamming distance over $\Hamming_n$, a set $X\subseteq \Hamming_n$ is \emph{$k$-convex} for some $k \geq 0$ integer if for all $x,y \in X$ with $\dHamming(x, y) \leq k$ and for all $z \in \Hamming_n$, $z \in X$ whenever $\dHamming(x,y) = \dHamming(x,z) + \dHamming(z,y)$.

An \emph{(undirected) graph} is a pair $G = (V,E)$, where $V$ is a finite set of vertices and $E\subseteq \{\{u,v\} \subseteq V \}$ is a set of edges;
$V$, $E$, and an edge $\{x,y\} \in E$ will sometimes be denoted by $V(G)$, $E(G)$, and $xy$, respectively.
A graph $G'$ is a \emph{subgraph} of $G$ if $V(G') \subseteq V(G)$ and $E(G') \subseteq E(G)$. 
A \emph{path} of length $n$  for some $n \geq 0$ integer is a graph $P$ with $V(P) = \{v_1,\dots,v_n\}$ and $E(P) = \{v_i v_{i+1} : i \in [n-1]\}$ if $n > 0$; $E(P) = \emptyset$ otherwise.
The \emph{length} of a path is the number of edges it contains.
A graph is \emph{connected} if all pairs of its vertices are connected by a path. 
If two vertices of a graph $G$ are connected by a path, we define their \emph{geodesic distance} by the length of a shortest path connecting them. 
Note that it is a metric on the set of vertices for connected graphs. 
A subset $X \subseteq V(G)$ for a graph $G$ is \emph{geodesically convex} (or simply, \textit{convex}) if $V(P_{uv}) \subseteq X$ for \textit{all} $u, v \in V(G)$ and for \textit{all} shortest paths $P_{uv}$ connecting $u$ and $v$ \citep[see, e.g.,][]{Pelayo2013}. 
For $\theta \in \R_{\geq 0}$ and a finite metric space $(X,D)$, the \textit{$\theta$-neighborhood graph} is the graph $G$ with $V(G) = X$ and $E(G) = \{uv: u,v \in V(G) \text{ and } D(u,v) \leq \theta\}$.

For the standard definitions of \emph{concepts}, \emph{concept classes}, \emph{VC-dimension}, and \emph{polynomial PAC-learnability} from computational learning theory, the reader is referred to some standard text book \citep[see, e.g.,][]{Kearns1994}.
Let $\cC$ be a concept class over some domain $X$ and $k \in \N$.
The \emph{$k$-fold union of $C$} is defined by $\kfoldunion{\cC} = \{C_1 \cup \dots \cup C_k : C_i \in \cC \text{ for all } i \in [k] \}$. 
Note that the definition does \emph{not} require the $C_i$s to be pairwise distinct. 
The following problem is central to concept learning:
\begin{problem}[The Consistent Hypothesis Finding (CHF) Problem]
\label{problem:consistency}
\emph{Given} a concept class $\cC \subseteq 2^X$ over some domain $X$ and disjoint sets $E^+,E^-  \subseteq X$ of positive and negative examples, \emph{return} a concept $C \in \cC$ that is consistent with $E^+$ and $E^-$, i.e., $E^+ \subseteq C$ and $E^- \cap C = \emptyset$ if such a concept exists; otherwise \emph{return} the answer {\sc ``No''}.
\end{problem}
In order to prove polynomial PAC-learnability, we will use the following basic results from computational learning theory~\citep{Blumer1989}:
\begin{theorem}
\label{thm:BEHW}
Let $\cC  \subseteq 2^X$ be a concept class over some domain $X$ with VC-dimension $d > 0$.
\begin{enumerate}[label=(\roman*)]  
\item \label{thm:BEHW:PAC-if-VC-and-CHF}$\cC$ is polynomially PAC-learnable if $d$ is bounded by a polynomial of its parameters and Problem~\ref{problem:consistency} can be solved in time polynomial in the parameters and $\card{E^+\cup E^-}$.
\item \label{thm:BEHW:k-fold-unions} For the VC-dimension of $\cC_\cup^k$ we have $\vcdim\left(\cC_\cup^k\right) \leq 2dk\log(3k)$ for all $k \geq 1$.
\end{enumerate}
\end{theorem}

\section{Weak Convexity in Metric Spaces} 
\label{sec:weakconvexity} 
In this section, we introduce the notion of \emph{weak convexity} in metric spaces and establish some basic formal properties of weakly convex sets that will be utilized in the subsequent sections.
By the most common definition, a set $A \subseteq \Rd$ is \emph{convex}~\citep{Menger1928} if 
\begin{equation}
	\label{eq:triangleeq}
	\dist_2(x,z)+\dist_2(z,y) = \dist_2(x,y) \implies z \in A
\end{equation}
for all $x,y \in A$ and $z \in \Rd$. 
Our notion of weak convexity generalizes (\ref{eq:triangleeq}). 
It is motivated by the fact that convex sets defined by (\ref{eq:triangleeq}) are always ``contiguous'' and cannot therefore capture well-separated regions of the domain that are ``locally'' convex.\footnote{The notion of local convexity in this paper is different from the one used in topology.}   
We address this problem by \emph{adapting} the idea of $k$-convexity over Hamming metric spaces~\citep{Ekin1999} or that of $g_k$-convexity over graphs equipped with the geodesic distance~\citep{Farber1986}\footnote{The notion of $g_k$-convexity~\citep{Farber1986} in graph theory has been used to study different graph classes for which global convexity can be characterized by weak (or local) geodesic convexity \citep[see, e.g.,][]{Chalopin2022,Chepoi1989,Chepoi2007,Chepoi2020}.} to \emph{arbitrary} finite and complete infinite metric spaces. 
In particular, analogously to \citet{Ekin1999} and \citet{Farber1986}, we do \emph{not} require (\ref{eq:triangleeq}) to hold for all points $x$ and $y$, but only for such pairs which have a distance at most a user-specified threshold. 
In other words, while convexity is based on a \emph{global} condition resulting in a single ``contiguous'' region, our notion of weak convexity relies on a \emph{local} one, resulting in potentially several isolated regions, where the spread of locality is controlled by the above mentioned threshold. 
We will be interested in weakly convex hulls of finite sets. 
Since the convex hull of any bounded and closed, in particular, any finite subset of $\Rd$ is closed, we require weakly convex sets to be closed.
These considerations yield the following formal definition of \emph{weak convexity}:
\begin{definition} 
\label{def:weak-convexity} 
	Let $(X,D)$ be a complete metric space. A set $A \subseteq X$ is \emph{$\theta$-convex} (or simply, \emph{weakly convex}) for some $\theta \in \R_{\geq 0}$ if $A$ is closed (i.e., $\cl(A) = A$) and for all $x,y \in A$ and $z \in X$ it holds that $z \in A$ whenever $\dist(x, y) \leq \theta$ and $z \in \ti(x,y)$, where
	\begin{align} \label{eq:almost-triangle-equality-condition}
		\ti(x,y) = \{z\in X: \dist(x, z) + \dist(z, y) = \dist(x, y) \} 
	\end{align}
	denotes the \textit{interval} of the points lying between $x$ and $y$.
\end{definition}
Notice that (\ref{eq:almost-triangle-equality-condition}) implies $x,y \in \ti(x,y)$. Furthermore, it does not require $x \neq y$. 
In particular, $\ti(x,x) = \{x\}$ for all $x \in X$. 
The family of all weakly convex sets is denoted by $\WConvSets[\dist]$; we omit $\dist$ if it is clear from the context. 
The definitions imply $\WConvSets[\dist]<0> = 2^X$.

\begin{figure}[tp]
	\color{black} 
	\centering
	\def\thisscale{0.45}
	\subfloat[small $\theta$]{\label{subfig:weak-convexity-in-R2:small}\begin{tikzpicture}[x=\thisscale cm,y=\thisscale cm]
		\draw[->] (-0.5,-0.75) -- (6,-0.75) node[right] {$x$}; 
		\draw[->] (0,-1.25) -- (0,6) node[above] {$y$}; 
		
		\coordinate (P1) at (1,1); 
		\coordinate (P2) at (2,0.5); 
		
		\coordinate (P3) at (4.75,1.15); 
		\coordinate (P4) at (5.75,1.75); 
		\coordinate (P5) at (5.2,2); 
		\coordinate (P6) at (5.6,2.25); 
		\coordinate (P7) at (4.3,2.5); 
		
		\coordinate (P8) at (3,5.5); 
		\coordinate (P9) at (2,5.1); 
		\coordinate (P10) at (2.5,5.75);
		
		\fill[black!15,dashed] (P8) -- (P10) -- (P9) -- (P1) -- (P2) -- (P3) -- (P4) -- (P6) -- cycle; 
		
		\draw[line width=1pt, tabred] (P10) -- (P8); 
		\fill[tabred] (P4) -- (P5) -- (P6) -- cycle; 
		
		\foreach \i in {1,...,10} {
			\fill (P\i) circle (1pt); 
		}; 
	\end{tikzpicture}} \quad
	\subfloat[medium $\theta$]{\label{subfig:weak-convexity-in-R2:medium}\begin{tikzpicture}[x=\thisscale cm,y=\thisscale cm]
		\draw[->] (-0.5,-0.75) -- (6,-0.75) node[right] {$x$}; 
		\draw[->] (0,-1.25) -- (0,6) node[above] {$y$}; 
		
		\coordinate (P1) at (1,1); 
		\coordinate (P2) at (2,0.5); 
		
		\coordinate (P3) at (4.75,1.15); 
		\coordinate (P4) at (5.75,1.75); 
		\coordinate (P5) at (5.2,2); 
		\coordinate (P6) at (5.6,2.25); 
		\coordinate (P7) at (4.3,2.5); 
		
		\coordinate (P8) at (3,5.5); 
		\coordinate (P9) at (2,5.1); 
		\coordinate (P10) at (2.5,5.75);

		\coordinate (Q1) at (3.375,0.825); 
		
		\fill[black!15,dashed] (P8) -- (P10) -- (P9) -- (P1) -- (P2) -- (P3) -- (P4) -- (P6) -- cycle; 
		\draw[line width=1pt,tabred] (P1) -- (P2) node[black,midway,below] {$A_1$}; 
		\fill[tabred] (P3) -- (P4) -- (P6) -- (P7) node[black,midway,above right] {$A_3$} -- cycle; 
		\fill[tabred] (P8) -- (P9) -- (P10) node[black,midway,above left] {$A_2$} -- cycle; 
		\draw[black!50,|-|] ([shift={(0.1625,-0.6875)}]P2) -- ([shift={(0.1625,-0.6875)}]P3) node[midway,below] {$> \theta$}; 

		\node[below right=-0.5mm] at (P2) {$x$};
		\node[below left=-0.5mm] at (P3) {$y$}; 
		\draw (Q1) circle (1pt) node[above] {$z$};  

		\foreach \i in {1,...,10} {
			\fill (P\i) circle (1pt); 
		}; 
	
	\end{tikzpicture}} \quad 
	\subfloat[large $\theta$]{\label{subfig:weak-convexity-in-R2:large}\begin{tikzpicture}[x=\thisscale cm,y=\thisscale cm]
		\draw[->] (-0.5,-0.75) -- (6,-0.75) node[right] {$x$}; 
		\draw[->] (0,-1.25) -- (0,6) node[above] {$y$}; 
		
		\coordinate (P1) at (1,1); 
		\coordinate (P2) at (2,0.5); 
		
		\coordinate (P3) at (4.75,1.15); 
		\coordinate (P4) at (5.75,1.75); 
		\coordinate (P5) at (5.2,2); 
		\coordinate (P6) at (5.6,2.25); 
		\coordinate (P7) at (4.3,2.5); 
		
		\coordinate (P8) at (3,5.5); 
		\coordinate (P9) at (2,5.1); 
		\coordinate (P10) at (2.5,5.75);

		\coordinate (Q1) at (3.375,0.825); 
		
		\fill[tabred,dashed] (P8) -- (P10) -- (P9) -- (P1) -- (P2) -- (P3) -- (P4) -- (P6) -- cycle; 
		
		\foreach \i in {1,...,10} {
			\fill (P\i) circle (1pt); 
		}; 
		
		\end{tikzpicture}}
	\caption{Examples of $\theta$-convex sets in $\R^2$ for different values of $\theta$.}
	\label{fig:weak-convexity-in-euclidean-plane} 
\end{figure}

To illustrate the notion of weak convexity, consider the finite set of points $A \subseteq \R^2$ in Figure~\ref{subfig:weak-convexity-in-R2:medium}. 
While the convex hull of $A$ is indicated by the gray and red areas, the $\subseteq$-smallest $\theta$-convex set containing $A$ for some suitable $\theta \geq 0$ is drawn in red. The black points also belong to the $\theta$-convex sets in \ref{subfig:weak-convexity-in-R2:small}--\ref{subfig:weak-convexity-in-R2:large}.
The most obvious difference is that there are three separated regions $A_1,A_2$, and $A_3$, instead of a single contiguous area. 
In other words, in contrast to convex sets in $\R^2$, weakly convex sets need \textit{not} be connected. 
This is a consequence of considering only point pairs with distance at most $\theta$. 
For example, the points $x$ and $y$ in Figure~\ref{subfig:weak-convexity-in-R2:medium} have a distance strictly greater than $\theta$, implying that they do not generate $z$. 
Note that in the same way as convex sets, (parts of) weakly convex sets may be degenerated. 
For example, while $A_2$ and $A_3$ are regions with strictly positive area, $A_1$ is just a segment. 
We may even have isolated points (see Figure~\ref{subfig:weak-convexity-in-R2:small}). For sufficiently large $\theta$, the $\subseteq$-smallest $\theta$-convex set containing $A$ becomes equal to the convex hull of $A$ (cf. \ref{subfig:weak-convexity-in-R2:large}).

Despite this unconventional behavior of weakly convex sets, $(X,\WConvSets)$ forms a \emph{closure system}. 
To see this, note  that $\emptyset,X \in \WConvSets$. 
Let $\mathcal{F} \subseteq \WConvSets$ and $x,y \in \bigcap \mathcal{F}$ with $\dist(x, y) \leq \theta$.
Then $\ti(x, y) \subseteq F$ for all $F \in \mathcal{F}$ implying that $\bigcap \mathcal{F}$ is $\theta$-convex. 
Thus, $\WConvSets$ has an associated \emph{closure operator} $\wconv: \power{X} \to \power{X}$ with $A \mapsto \bigcap \{C \in \WConvSets : A \subseteq C\}$ for all $A \subseteq X$. 
That is, $\wconv$ maps a set $A$ to the $\subseteq$-smallest $\theta$-convex set containing $A$. 
It is called the \emph{weakly convex hull operator} and its fixed points (i.e., the $\wconv$-closed sets) form exactly $\WConvSets$. 

\subsection{Some Basic Properties of Weakly Convex Sets}
\label{sec:weakly-convex-properties}
We now present some basic properties of weakly convex sets that make them especially interesting for machine learning from a practical as well as from a theoretical viewpoint. 
%
%
As already mentioned, weakly convex sets need \emph{not} be ``contiguous'' (cf. Figure~\ref{fig:weak-convexity-in-euclidean-plane}), in contrast to, e.g., convex sets in Euclidean spaces. 
Instead, one can observe regions that are \textit{separated} from each other, due to the fact that 
weak convexity utilizes a distance threshold $\theta$. 
In Theorem~\ref{thm:weak-convex-decomposition} below 
we formally state this property of weakly convex sets.
We note that this result generalizes that for the Hamming metric space \citep[cf. Proposition 3.2 in][]{Ekin2000} to complete metric spaces.

We first introduce some necessary notions. 
Let $\M= (X,\dist)$ be a metric space, $\theta \geq 0$, and $C \subseteq X$.
Two points $a,b \in C$ are \emph{$\theta$-connected} in $C$, denoted $a \connected b$, if there is a finite sequence $a = p_1, p_2, \dots, p_r = b \in C$ such that $\dist(p_i, p_{i+1}) \leq \theta$ for all $i \in [r-1]$. 
$C$ is \emph{$\theta$-connected} if $a \connected b$ for all $a, b \in C$. 
Note that $\connected$ is an equivalence relation on $C$; the equivalence class of $a$ is denoted by $\conncomp{a}$ (i.e., $\conncomp{a} = \{b \in C: a \connected b\}$) for all $a \in C$.

\begin{theorem} \label{thm:weak-convex-decomposition} 
	Let $(X,\dist)$ be a complete metric space, $\theta \geq 0$, and $C$ be a subset of $X$ that is finite if $\theta =0$. 
	Then $C$ is $\theta$-convex if and only if there is a unique family of non-empty sets $(B_i \subseteq C)_{i \in I}$ for some index set $I$ that satisfies the following conditions: 
	\begin{enumerate}[label=(\roman*)] 
		\item \label{proof-enum:weak-convex-disjoint-union-components}$C = \bigcup_{i \in I} B_i$, 
		\item \label{proof-enum:components-weakly-convex}$B_i$ is $\theta$-convex for all $i \in I$, 
		\item \label{proof-enum:components-theta-connected}$B_i$ is $\theta$-connected for all $i \in I$, 
		\item \label{proof-enum:components-distance} for all $i,j \in I$ with $i \neq j$, $\dist(a, b) > \theta$ for all $a \in B_i, b \in B_j$.
	\end{enumerate}
\end{theorem}
\begin{proof}
    The case $\theta = 0$ is trivial, so we assume $\theta > 0$. 
	We first show the equivalence stated in the theorem.
	For the ``if'' direction, suppose conditions~\ref{proof-enum:weak-convex-disjoint-union-components}--\ref{proof-enum:components-distance} hold for a family $(B_i)_{i \in I}$. 
	To show that $C$ is $\theta$-convex, let $x, y \in C$ with $\dist(x,y) \leq \theta$. 
	Then, by \ref{proof-enum:weak-convex-disjoint-union-components} and \ref{proof-enum:components-distance}, $x,y \in B_i$ for some $i \in I$. 
	Let $z \in \ti(x, y)$. 
	By \ref{proof-enum:components-weakly-convex} we have $z \in B_i$ and thus $z \in C$ by \ref{proof-enum:weak-convex-disjoint-union-components}. Any convergent sequence in $C$ is almost completely contained in some $B_i$ for $i \in I$; this follows from \ref{proof-enum:components-distance} and $\theta > 0$. 
    Since $B_i$ is closed, its limit point is also contained in $B_i$ and therefore in $C$ by \ref{proof-enum:weak-convex-disjoint-union-components}. Hence, $C$ is $\theta$-convex. 
	
	For the ``only if'' direction, assume that $C$ is $\theta$-convex. 
	Let $(a_i \in C)_{i \in I}$ denote a complete set of representatives of $\connected$ for some index set $I$. Let $B_i = \conncomp{a_i}$ for all $i \in I$. 
	  By construction, $(B_i)_{i \in I}$ satisfies \ref{proof-enum:weak-convex-disjoint-union-components},  \ref{proof-enum:components-theta-connected}, and \ref{proof-enum:components-distance}.
	In particular, \ref{proof-enum:components-distance} follows from $\theta > 0$ and the fact that for all $a,b \in C$, $\dist(a,b) \leq \theta$ implies $\conncomp{a} = \conncomp{b}$. 
	Thus, $\dist(a,b) > \theta$ for all $i \neq j, a \in B_i$, and $b \in B_j$.	
	To see that $B_i$ fulfills \ref{proof-enum:components-weakly-convex}, let $x, y \in B_i$ for some $i \in I$ such that $\dist(x, y) \leq \theta$, and let $z \in \ti(x, y)$. 
	Suppose for contradiction, that $z \notin B_i$. 
	Then by the $\theta$-convexity of $C$ and \ref{proof-enum:weak-convex-disjoint-union-components} we have that $z \in B_j$ for some $j \neq i$ and by \ref{proof-enum:components-distance} that  $\dist(x,z),\dist(z,y) > \theta$. 
	Therefore, 
	\[ 0 = \dist(x,z) + \dist(z,y) - \dist(x,y) > \theta \enspace , \]
	contradicting $\theta > 0$.
	Hence, $z \in B_i$. 
    Finally, since $C$ is closed, every convergent sequence in $B_i$ has a limit point by \ref{proof-enum:weak-convex-disjoint-union-components}, which lies in $B_i$ by \ref{proof-enum:components-distance} and $\theta > 0$. This completes the proof of \ref{proof-enum:components-weakly-convex}.
	
	It remains to show that $(B_i)_{i \in I}$ is \emph{unique} with respect to \ref{proof-enum:weak-convex-disjoint-union-components}--\ref{proof-enum:components-distance}. 
	Let $(B'_j)_{j \in J}$ be a family of non-empty sets that satisfies \ref{proof-enum:weak-convex-disjoint-union-components}--\ref{proof-enum:components-distance}. 
    Let $r \in J$, $x \in B'_r$, and $s \in I$ such that $x \in B_s$. 
    We claim that $B'_r = B_s$. 
    We only show $B'_r \subseteq  B_s$; the proof of $B_s \subseteq  B'_r$ is analogous.
	Suppose for contradiction that $B'_r \nsubseteq  B_s$ and let $a \in B'_r \setminus B_s$. 
	Then, by \ref{proof-enum:components-theta-connected}, there is a finite sequence $x = p_1,\dots,p_t = a \in B'_r$ with $\dist(p_i,p_{i+1}) \leq \theta$ for all $i \in [t-1]$. 
	It must be the case that there is an $i \in [t-1]$ such that $p_i \in B_s$ and $p_{i+1} \in B'_r \setminus B_s$. 
	But then, since $p_{i+1} \notin B_s$, $\dist(p_i,p_{i+1}) > \theta$ because $(B_i)_{i \in I}$ satisfies \ref{proof-enum:weak-convex-disjoint-union-components} and \ref{proof-enum:components-distance}, which is a contradiction. 
	Hence, $B'_r = B_s$.
	Thus, for all $j \in J$ there is an $i \in I$ such that $B'_j = B_i$, implying the uniqueness. 
\end{proof}
In what follows, the family $(B_i)_{i \in I}$ satisfying conditions~\ref{proof-enum:weak-convex-disjoint-union-components}--\ref{proof-enum:components-distance} in Theorem~\ref{thm:weak-convex-decomposition} will be referred to as the \emph{$\theta$-decomposition} of the $\theta$-convex set $C$. 
Furthermore, the sets $B_i$ in the $\theta$-decomposition of $C$, denoted $\blocks<\theta>(C)$, will be called \emph{$\theta$-blocks} or simply, \emph{blocks}.
We will omit $\theta$ from the notation and simply write $\blocks(C)$ if $\theta$ is clear from the context.
In particular, we write $\blocks(\wconv(A))$ instead of $\blocks<\theta>(\wconv(A))$ for any $A \subseteq X$. 
In Proposition~\ref{pr:distance_existence} below we formulate a basic property of the blocks in a $\theta$-decomposition.   
\begin{proposition} 
	\label{pr:distance_existence}
	Let $\M=(X,\dist)$ be a complete metric space, $\theta \geq 0$, and $A \in \finitesubsets$.
	Then for all $B_1,B_2\in \blocks(\wconv(A))$, there are $a \in B_1$ and $b \in B_2$ such that
	$$\dist(B_1,B_2) = D(a,b)\enspace.$$
\end{proposition}
\begin{proof}
It follows from the property that all blocks are non-empty and $\theta$-convex by Theorem~\ref{thm:weak-convex-decomposition} and hence closed.
\end{proof}

Finally, we claim that the weakly convex hull operator is monotone with respect to $\theta$ and establish a connection between weak and ordinary 
convexity. 
For a metric space $(X,\dist)$ and $A\subseteq X$, let $\conv(A)$ denote the closed convex hull of $A$.
\begin{proposition}
	\label{prop:monotonicity-of-hulls-wrt-parameters} 
	Let $(X,\dist)$ be a complete metric space, $0 \leq \theta \leq \theta' < \infty$, and let $A$ be a subset of $X$ that is finite if $\theta = 0$. Then
	\begin{enumerate}[label=(\roman*)]
		\item \label{prop:monotonicity:monotone-wrt-theta} $\wconv(A) \subseteq \wconv<\theta'>(A) \subseteq \conv(A)$,
		\item \label{prop:monotonicity:block-cooccurrence} for all $x,y \in \wconv(A)$, $x, y$ are in the same block of $\wconv<\theta'>(A)$ if they are in the same block of $\wconv(A)$.
	\end{enumerate} 
\end{proposition}
\begin{proof} 
    The claim is trivial if $\theta = 0$. Assume that $\theta > 0$.
	Regarding \ref{prop:monotonicity:monotone-wrt-theta}, the proof of the second containment is trivial.
The first one follows from the fact that the family of $\theta'$-convex sets containing $A$ is a subfamily of that of $\theta$-convex sets containing $A$. 
Indeed, any $\theta'$-convex set $C$ is $\theta$-convex, as $\ti(x, y) \subseteq C$ for all $x, y \in C$ with $\dist(x, y) \leq \theta \leq \theta'$. 
	To prove \ref{prop:monotonicity:block-cooccurrence}, let $x,y \in \wconv(A)$ that belong to the same $\theta$-block in $\wconv(A)$. 
	Then $\theta > 0$ and \ref{proof-enum:components-theta-connected} of Theorem~\ref{thm:weak-convex-decomposition} imply $x \connected<\theta>[\wconv(A)] y$.
	By $\theta \leq \theta'$ and \ref{prop:monotonicity:monotone-wrt-theta}, we also have $x \connected<\theta'>[\wconv<\theta'>(A)] y$. 
    Hence, $x, y$ lie in the same $\theta'$-block of $\wconv<\theta'>(A)$. 
\end{proof} 

\begin{figure}[tp]
    \centering
    \subfloat[$\theta=8$]{
        \centering
        \label{subfig:illustrative-large-graph:small-theta}
        \includegraphics[width=0.3\linewidth]{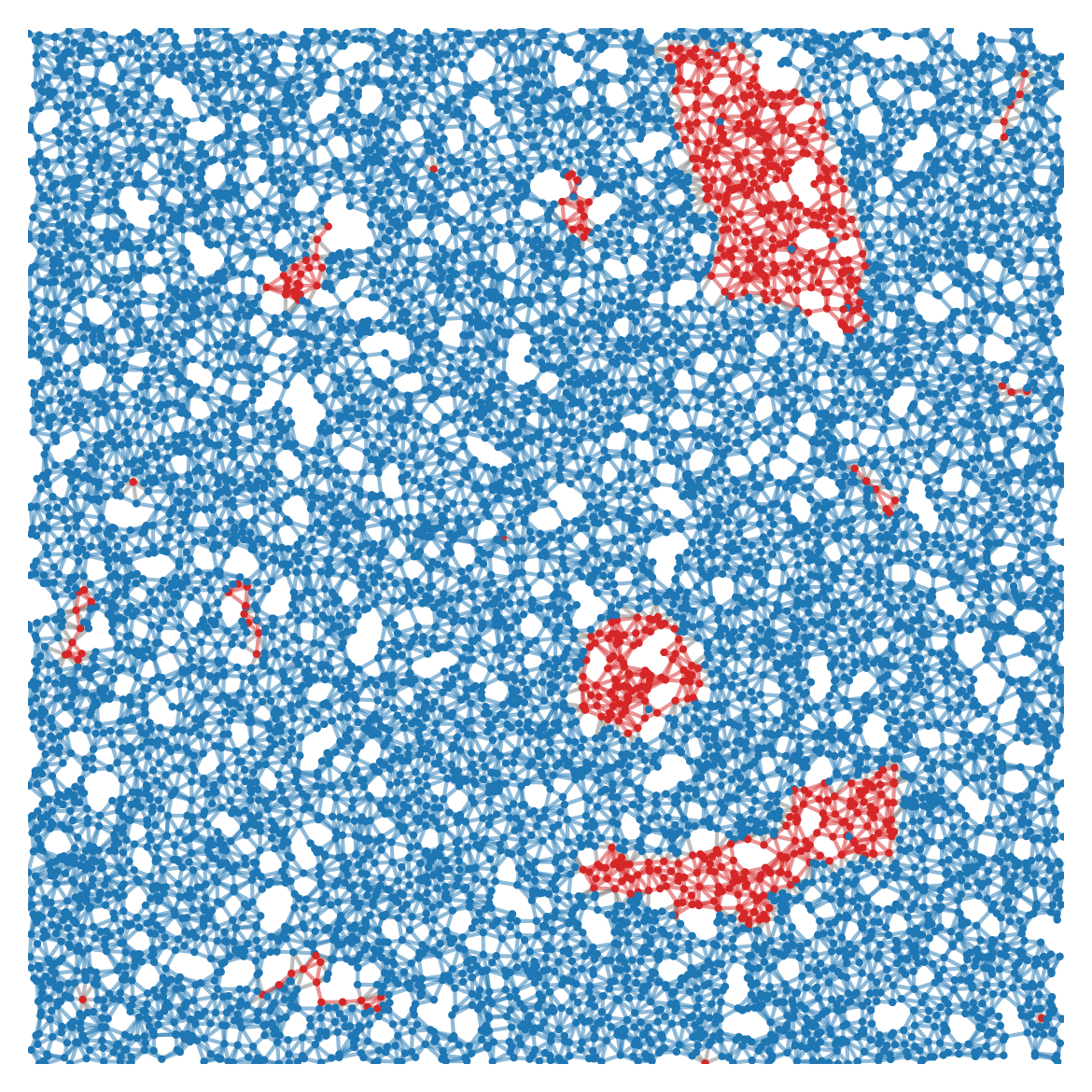}
    }
    \subfloat[$\theta = 15$]{
        \centering
        \label{subfig:illustrative-large-graph:medium-theta}
        \includegraphics[width=0.3\linewidth]{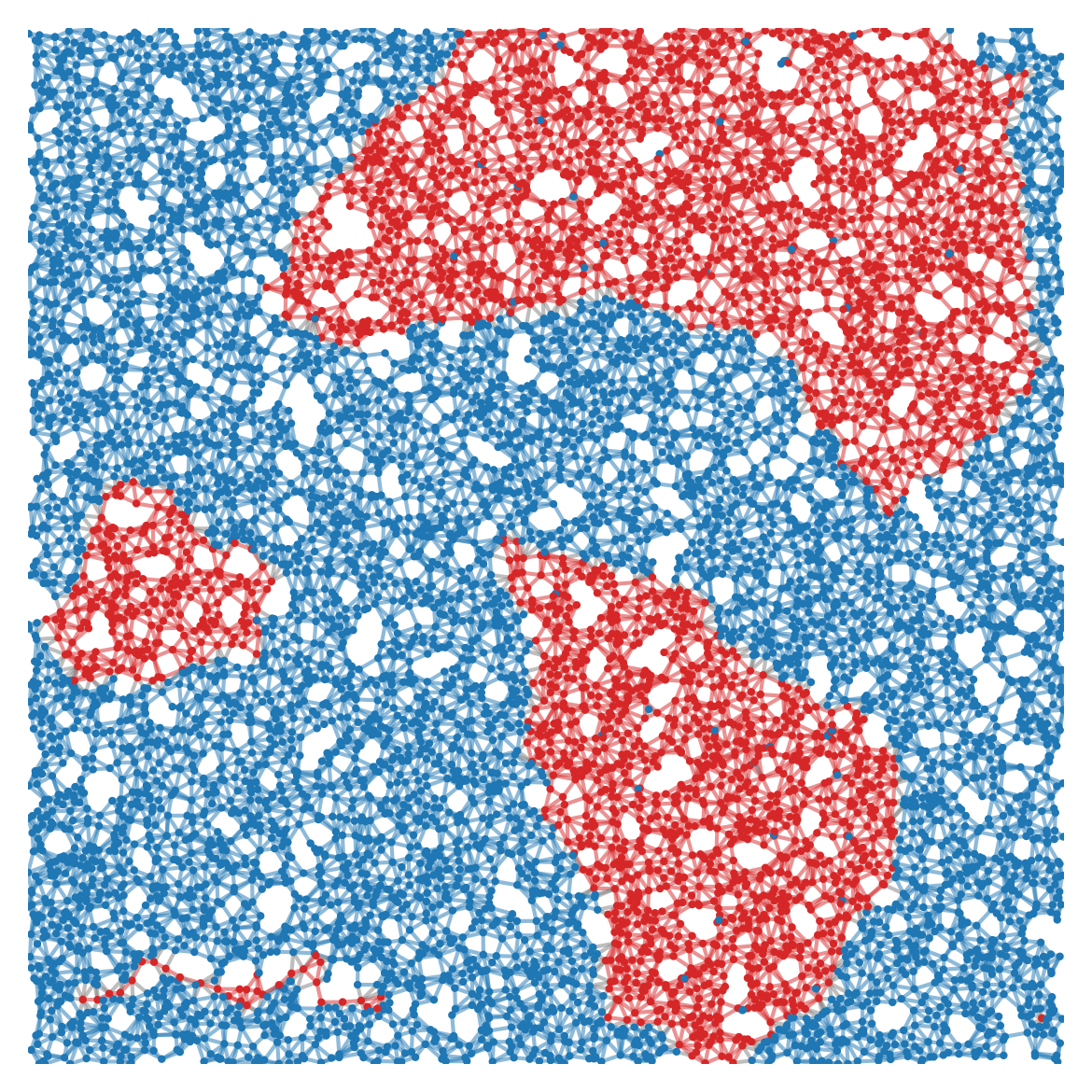}
    } 
    \subfloat[$\theta = 114$]{
        \centering
        \label{subfig:illustrative-large-graph:large-theta}
        \includegraphics[width=0.3\linewidth]{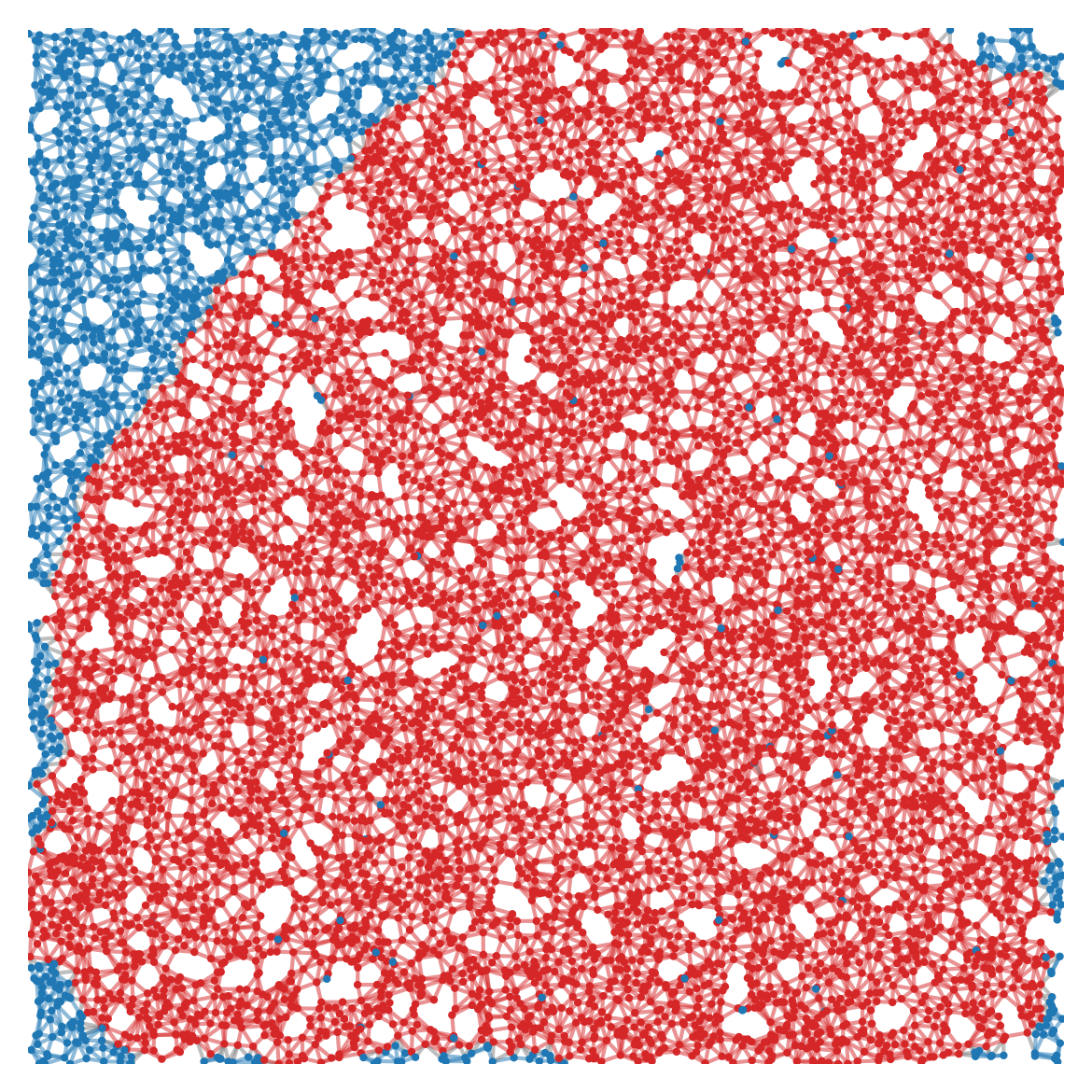}
    }
    \caption{The geodesic $\theta$-convex hulls (in \textcolor{tabred}{\textbf{red}}) of a set of $40$ vertices for $\theta=8, 15, \text{ and } 114$ in a graph with $10,000$ vertices.
    }
    \label{fig:illustrative-large-graph}
\end{figure}

Thus, for monotonically increasing $\theta$s, the weakly convex hulls of a set $A$ form a \textit{monotone chain}, with a maximum element defined by the convex hull of $A$.
In Figure~\ref{fig:illustrative-large-graph} we present an example of this property for the case that the domain is the vertex set of a graph and the distance is the geodesic (or shortest-path) distance.
The graph in this example, used also in our experimental evaluation, has 10,000 vertices.
We show three weakly convex subsets (color red) for $\theta = 8, 15, \text{ and } 114$ (the diameter of the graph).
They form the weak convex hulls of the same set of 40 vertices. 
A closer look at the figure shows that the weakly convex hull in (a) is a subset of that in (b), which, in turn, is a subset of the convex hull in (c), in accordance with Proposition~\ref{prop:monotonicity-of-hulls-wrt-parameters}. 
This property will be utilized in the next sections concerning learning weakly convex sets.

\begin{remark}
\label{rem:circle-not-blockwise-convex} 
We note that a block of a weakly convex set is \emph{not} necessarily convex. 
To see this, let $\theta \geq 2$ be an integer and consider the graph $C$ consisting of a single cycle of length $2\theta+3$ with the geodesic distance $D$ as metric. 
Let $A=\{v_1,\ldots,v_k\} \subseteq V(C)$ such that $D(v_i,v_{i+1}) \leq \theta$ for all $1\leq i< k$, $D(v_1,v_k) = \theta+1$, and the shortest path $P_{v_1,v_k}$ connecting $v_1$ and $v_k$ does not contain any vertices from $A \setminus \{v_1,v_k\}$. 
It follows that $\wconv(A)$ consists of a single block. 
Furthermore, while $\conv(A) = V(C)$, $\wconv(A)$ does not contain the interior points of $P_{v_1,v_k}$.
Thus, $\wconv(A)$ is not convex.
\end{remark}

\section{The Generic CHF Algorithm} 
\label{sec:intensional-setting} 
In this section we present our general algorithm for solving the CHF problem for weakly convex hypothesis classes over a  broad class of metric spaces.
As long as $\theta$ is fixed, the CHF problem for a $\theta$-convex hypothesis class $\cC_\theta$ can be solved by computing the $\theta$-convex hull of the positive examples.
If, however, $\theta$ is not given in advance, which is a realistic scenario, we need to compute a solution from those of the CHF problems for $\cC_\theta$ for all $\theta \geq 0$ that is \textit{optimal} with respect to some criterion.
Such an optimality criterion could be defined by the number of blocks in the weakly convex hull. 
This criterion leads, however, to solving a computationally intractable problem (see Section~\ref{sec:conclusion} for a discussion). 
We therefore restrict the set of feasible solutions to the $\theta$-convex hulls of the positive examples for all $\theta \geq 0$ and return the \textit{most general} consistent hypothesis~\citep{Mitchell1982} defined by the \textit{largest}  weakly convex hull of the positive examples for some $\theta$ that is disjoint with the negative examples. 
Out of the consistent weakly convex hulls, it is the closest approximation of the convex hull of the positive examples.


We consider the case that the hypotheses are given by some  representation and refer to this scenario as the \textit{intensional} problem setting.
In order to formulate the CHF problem, we first  define the notion of representation schemes.  
The definition below utilizes that  weakly convex sets give rise to a unique block decomposition (cf. Theorem~\ref{thm:weak-convex-decomposition}).
More precisely, let $\M= (X,\dist)$ be a complete metric space and $\tau \geq 0$. 
A \textit{representation scheme} for  $\M$ and $\tau$ is a function $\repwconv: \R_{\geq \tau} \times \finitesubsets \to 2^{\{0,1\}^*}$ with
\begin{equation}
	\label{eq:representation_mu}
	\repwconv[A] \mapsto \{\blockrepwconv[B \cap A]: B \in \blocks(\wconv(A))\}
\end{equation}
for all $\theta \geq\tau$ and $A \in \finitesubsets$, where 
$\blockrepwconv: \R_{\geq \tau} \times \finitesubsets \to \{0,1\}^* \cup \{\bot\}$ is a function such that for all $\theta',\theta'' \geq \tau$ and for all $A',A'' \in \finitesubsets$ it satisfies
\begin{align}
	\blockrepwconv<\theta'>[A'] \in \{0,1\}^* &\iff \card{\blocks(\wconv<\theta'>(A'))}=1 
	\label{enum:singleton-block-representations}
	\intertext{and if $|\blocks(\clop_{\theta'}(A'))| = |\blocks(\clop_{\theta''}(A''))| =1$ then}
	\wconv<\theta'>(A') = \wconv<\theta''>(A'') & \iff \blockrepwconv<\theta'>[A'] = \blockrepwconv<\theta''>[A''] \enspace .
	\label{enum:block-representations-coherent}
\end{align}
In other words, $\repwconv$ returns some \textit{unique} representation of weakly convex hulls of \textit{finite} subsets of the domain using some representation $\blockrepwconv$ of weakly convex blocks.
The definition above is correct, as $\blockrepwconv[B \cap A] \in \{0,1\}^*$ for all blocks $B$ in (\ref{eq:representation_mu}).
For $R = \repwconv[A]$, the extension of $R$ (i.e., $\wconv(A)$) is denoted by $\ext(R)$.  
For $\M$, $\tau$, and $\repwconv$, define the order $\preccurlyeq$ on the set of representations of weakly convex sets as follows: For all $A,B \in \finitesubsets$ and for all $\theta_1,\theta_2 \geq \tau$, $\repwconv<\theta_1>[A] \preccurlyeq \repwconv<\theta_2>[B]$ if and only if $\wconv<\theta_1>(A) \subseteq \wconv<\theta_2>(B)$.  
Clearly, $\preccurlyeq$ is a partial order.
Using the above notions, we are ready to define the CHF problem for the intensional setting. 
The supremum in the definition below is taken with respect to the relation $\preccurlyeq$. 
\begin{problem}
	\label{problem:intCHF}
	\textit{Given} a complete metric space $\M= (X,\dist)$, a representation scheme $\repwconv$ for $\M$ and some $\tau \geq 0$, and disjoint sets $E^+,E^- \in \finitesubsets$ of labeled examples with $E^+ \neq \emptyset$,
	\textit{return} 
	$$
	\sup_{\theta \geq \tau} \{\repwconv[E^+] : \wconv(E^+) \cap E^- = \emptyset \} 
	$$
	if such a $\theta$ exists; otherwise return ``\symbolno{}''.
\end{problem}

To present our solution to Problem~\ref{problem:intCHF}, we need a restriction on complete metric spaces.
Section~\ref{sec:k-convexity} below is concerned with learning weakly convex Boolean functions in the Hamming metric space $(\Hamming_n,\dHamming)$. 
For $\theta = 1$, \textit{all} subsets 
of $\Hamming_n$ are $1$-convex. 
Thus, to represent any of the $2^{2^n}$ $1$-convex subsets, we need $\Omega(2^n)$ bits, implying that there is \textit{no} compact representation of $1$-convex sets.
One of the problems is that the blocks of $1$-convex subsets of $\Hamming_n$ are not convex in general.
%
To overcome this problem, we require, in addition to completeness, the metric space to satisfy the blockwise convexity property, defined as follows:
A metric space $\M=(X,\dist)$ is \textit{blockwise convex} for some $\tau \geq 0$, if it is complete and for all $\tau$-convex sets $C \subseteq X$ with $C = \wconv<\tau>(A)$ for some $A \in \finitesubsets$, $C$ is \textit{convex} whenever it is $\tau$-connected. In other words, all blocks of the $\tau$-convex hull of a finite set are convex. 
The definitions imply that if $\M$ is blockwise convex for some $\tau \geq 0$, then it is blockwise convex for all $\tau' \geq \tau$.
In the lemma below we first present some basic properties of of weakly convex hulls in blockwise convex metric spaces.
\begin{lemma}
	\label{lm:wellbehaved_and_blockwiseconvex}
	Let $\M=(X,D)$ be a blockwise convex metric space for some $\tau \geq 0$, $A \in \finitesubsets$ with $A \neq \emptyset$, $\theta_0 = 0$ and $\theta_{i} = \max\{\tau,\min\{D(B_1,B_2) > 0: B_1,B_2 \in \blocks(\wconv<\theta_{i-1}>(A))\}\}$ for all $i \in[k]$, where $k$ is the smallest integer satisfying $\card{\blocks(\wconv<\theta_k>(A))}=1$. 
	Then $k \leq \card{A}$ and for all $i \in [k-1]$,
	\begin{enumerate}[label=(\roman*)]
		\item \label{lm:wellbehaved_and_blockwiseconvex:next-level-is-different} $\wconv<\theta_i>(A)\subseteq \wconv<\theta_{i+1}>(A)$,
		\item \label{lm:wellbehaved_and_blockwiseconvex:same-level-is-same} $\wconv<\theta_i>(A)=\wconv<\theta_i'>(A)$ for all $\theta_i' \in [\theta_i,\theta_{i+1})$, 
		\item \label{lm:wellbehaved_and_blockwiseconvex:convex-is-last} $\wconv<\theta_k>(A)=\wconv<\theta_k'>(A)$ for all $\theta_k' \geq \theta_k$ \enspace .
	\end{enumerate}
\end{lemma}
\begin{proof}
	We first prove that the $\theta_i$s in the claim fulfill \ref{lm:wellbehaved_and_blockwiseconvex:next-level-is-different}--\ref{lm:wellbehaved_and_blockwiseconvex:convex-is-last} for all $i\in[k]$.
	Let $i \in [k-1]$. 
    The proof of \ref{lm:wellbehaved_and_blockwiseconvex:next-level-is-different} follows directly from Proposition~\ref{prop:monotonicity-of-hulls-wrt-parameters} because $\theta_i < \theta_{i+1}$. 
	Regarding \ref{lm:wellbehaved_and_blockwiseconvex:same-level-is-same}, 
	we claim that $\wconv<\theta_i>(A)$ satisfies all conditions of Theorem~\ref{thm:weak-convex-decomposition} for all $\theta_i' \in [\theta_i,\theta_{i+1})$, and hence, $\wconv<\theta_i>(A)$ is $\theta_i'$-convex. This implies $\wconv<\theta_i>(A) \supseteq \wconv<\theta_i'>(A)$, from which we get \ref{lm:wellbehaved_and_blockwiseconvex:same-level-is-same} by \ref{lm:wellbehaved_and_blockwiseconvex:next-level-is-different}.
	To show this claim, note that \ref{proof-enum:weak-convex-disjoint-union-components} and \ref{proof-enum:components-theta-connected} of Theorem~\ref{thm:weak-convex-decomposition} hold trivially, \ref{proof-enum:components-weakly-convex} by blockwise convexity, and \ref{proof-enum:components-distance} by $\theta_i' < \theta_{i+1}$, together with the definition of $\theta_{i+1}$.
	Finally, the proof of \ref{lm:wellbehaved_and_blockwiseconvex:convex-is-last} of the lemma is automatic, as $\card{\blocks(\wconv<\theta_k>(A))}=1$ and hence, it is convex by blockwise convexity. 
	The proof of $k \leq \card{A}$ follows from $\card{\blocks(\wconv<\theta_0>(A))} = \card{A}$ and from $\card{\blocks(\wconv<\theta_i>(A))} > \card{\blocks(\wconv<\theta_{i+1}>(A))}$ ($1 \leq i < k$).
\end{proof}

%
%
\begin{algorithm}[t]
	\begin{algorithmic}[1]
		\Require blockwise convex metric space $\M= (X,\dist)$ for some $\tau \geq 0$ and representation scheme $\repwconv$ for $\M$ and $\tau$
		\Input disjoint sets $E^+,E^- \in \finitesubsets$ with $E^+ \neq \emptyset$	
		\Output $\repwconv[E^+]$ such that $\theta \geq \tau$, $\wconv(E^+)\cap E^- = \emptyset$, and $\wconv<\theta'>(E^+)\cap E^- \neq \emptyset$ for all $\theta' > \theta$ satisfying $\wconv<\theta'>(E^+) \supsetneq \wconv(E^+)$ if such a $\theta$ exists; ``\symbolno{}'' otherwise
		\Statex 
		\State $\decomp_0 \gets \{\Call{Singleton}{x} : x \in E^+\}$, $i \gets 0$ \hfill // cf. (\ref{eq:singleton})
    \label{line:init}
		\While{$\card{\decomp_{i}} > 1$}
        \label{line:start_outer_loop}
		\State $i \gets i+1$ 
		\State $\theta_{i} = \max\{\tau,\min \{\fn{Distance}(R_1,R_2): R_1,R_2 \in \decomp_{i-1}, \ R_1 \neq R_2\}\}$ 		\hfill // cf. (\ref{eq:distance}) 
        \label{line:next_theta}
		\State $\decomp \gets \decomp_{i-1}$ 
        \label{line:set_R}
		\While{$\exists R_1,R_2 \in \decomp$ with  $R_1 \neq R_2$ and $\Call{Distance}{R_1,R_2} \leq \theta_{i}$}
    \label{line:inner_while_loop}
		\State {$\decomp \gets (\decomp \setminus \{R_1,R_2\}) \cup \{R\}$ with $R = \Call{Join}{\theta_{i},R_1,R_2}$} \hfill // cf. (\ref{eq:join})	
        \label{line:join}
		\EndWhile
        \label{line:end_inner_while_loop}
        \State {$\decomp_i \gets \decomp$}
        \label{line:set_R_i}
		\If{$\exists e \in E^-$ and $R\in \decomp_i \setminus \decomp_{i-1}$ such that $\fn{Membership}(e,R) = \symboltrue{}$}
        \label{line:start_consistency}
		\If{$i=1$ and $\theta_i = \tau$}
		\Return ``\symbolno''
		\Else \	\Return $\decomp_{i-1}$
		\EndIf		
		\EndIf
        \label{line:end_consistency}
		\EndWhile
        \label{line:end_outer_loop}
		\State \Return $\decomp_i$ 
		%
	\end{algorithmic}
	\caption{\sc Intensional Consistent Hypothesis Finding} 
	\label{alg:intCHF}
\end{algorithm}
%
We are ready to present our general domain-independent algorithm (see Algorithm~\ref{alg:intCHF}) for solving Problem~\ref{problem:intCHF} for blockwise convex metric spaces. 
It utilizes the property that the $\theta$-convex hulls of the positive examples form an ascending chain for increasing $\theta$s and hence, the representations of any two weakly convex hulls are comparable with respect to $\preccurlyeq$ defined above.
%

For disjoint finite sets $E^+,E^- \subseteq X$ of examples, Algorithm~\ref{alg:intCHF} first computes in $\decomp_0$ the set of representations of the singleton blocks containing $x$ for all $x \in E^+$ (line~\ref{line:init}), where function \fn{Singleton} is defined by
\begin{equation}
	\label{eq:singleton}
	\fn{Singleton}(x) = \blockrepwconv<0>[\{x\}] \enspace .
\end{equation}


In each iteration $i$ of the outer loop (lines~\ref{line:start_outer_loop}--\ref{line:end_outer_loop}), the algorithm computes in $\decomp_i$ a new set of block representations from those in $\decomp_{i-1}$ (lines~\ref{line:set_R}--\ref{line:set_R_i}).
In particular, it takes the \textit{smallest} pairwise distance $\theta_{i}$ between the blocks in $\decomp_{i-1}$ if it is greater than $\tau$; otherwise $\tau$ (cf. line~\ref{line:next_theta}). 
More precisely, function \fn{Distance} called in line~\ref{line:next_theta} with valid (block) representations $R_1,R_2 \in \{0,1\}^*$ is defined by
\begin{equation}
	\label{eq:distance}
	\fn{Distance}(R_1,R_2) = D(\ext(R_1),\ext(R_2)) 
	\enspace .
\end{equation}
We note that (\ref{eq:distance}) is the \textit{semantic} definition of \fn{Distance}; it is assumed that the distance between $\ext(R_1)$ and $\ext(R_2)$ can be calculated \textit{directly} from their representations $R_1,R_2$.

The algorithm then sets $\decomp$ to $\decomp_{i-1}$ and in the inner loop (lines~\ref{line:inner_while_loop}--\ref{line:end_inner_while_loop}) it iteratively joins \textit{all} pairs of block representations in $\decomp$ that have distance at most $\theta_{i}$, including also those that arise in the inner loop.
In Lemma~\ref{lm:IWCH} we claim that $\decomp_i = \repwconv<\theta_i>[E^+]$ for $\decomp_i$ in line~\ref{line:set_R_i}.
To prove this, in Proposition~\ref{pr:join} we first state that the $\theta$-convex hull of two blocks with distance at most $\theta$ is always a single block.
\begin{proposition}
	\label{pr:join}
	Let $\M=(X,D)$ be a complete metric space and $B_1,B_2 \in \blocks(\wconv<\theta'>(A))$ for some $A \in \finitesubsets$ and $\theta' \geq 0$. Let $\theta > 0$ such that $D(B_1,B_2)\leq \theta$.
	Then for $B = \wconv(B_1 \cup B_2)$ it holds that 
	\begin{enumerate}[label=(\roman*)]
		\item \label{pr:join:generators-determine-blocks} $B = \wconv(A\cap (B_1\cup B_2))$ and
		\item \label{pr:join:result-is-single-block} $B$ is a block (i.e., it is $\theta$-connected).
	\end{enumerate}
\end{proposition}	
\begin{proof}
	The claim is trivial if $\theta \leq \theta'$, as $B_1=B_2$ for this case by Theorem~\ref{thm:weak-convex-decomposition}. 
	Consider the case that $\theta > \theta'$. 
	The proof of \ref{pr:join:generators-determine-blocks} is straightforward.
	Since $B_1,B_2$ are $\theta'$-connected by Theorem~\ref{thm:weak-convex-decomposition}, $B_1\cup B_2$ is $\theta$-connected, from which \ref{pr:join:result-is-single-block} follows by noting that the $\theta$-convex hull of a $\theta$-connected set is $\theta$-connected.
\end{proof}
Using induction on $i$, Proposition~\ref{pr:join} implies that if $D(\ext(R_1),\ext(R_2))$ for $R_1,R_2$ in line~\ref{line:join} is at most $\theta_i$, then 
\begin{equation}
\label{eq:join_generators}
\wconv<\theta_i>(\ext(R_1) \cup \ext(R_2)) = \wconv<\theta_i>(E^+ \cap (\ext(R_1) \cup \ext(R_2)))    
\end{equation}
consists of a \textit{single} block, giving rise to the following definition of \fn{Join} in line~\ref{line:join}:
\begin{equation}
	\label{eq:join}
	\Call{Join}{\theta_i,R_1,R_2} = \blockrepwconv<\theta_i>[E^+ \cap (\ext(R_1)\cup\ext(R_2))]
	\enspace .
\end{equation}
Similarly to function \fn{Distance}, the algorithmic realization of \fn{Join} is assumed to operate directly on $R_1$ and $R_2$, and not on their extensions.
In the proof of Lemma~\ref{lm:IWCH} below, we will use the following auxiliary result for $R$ computed in line~\ref{line:end_inner_while_loop}: 
\begin{lemma}
	\label{lm:join}
	For all $i \geq 1$ and $R$ computed in line~\ref{line:end_inner_while_loop} in iteration $i$ of the outer loop, 
    there exists $R' \in \repwconv<\theta_i>[E^+]$ such that $\ext(R) \subseteq \ext(R')$. 
\end{lemma}
\begin{proof}
    We prove the claim by induction on the generation order of $R$ (cf. line~\ref{line:join}).  
    Suppose $R$ has been generated in iteration $i$ for some $R_1,R_2 \in \decomp$ and $i \geq 1$.
    The base case follows from \ref{pr:join:result-is-single-block} of Proposition~\ref{pr:join}, as $R_1,R_2 \in \decomp_0$ and hence their extensions are singletons.
    Suppose the statement holds for all blocks generated before $R$. 
    Let $i_1,i_2$ be the smallest indices such that $R_1 \in \decomp_{i_1}, R_2 \in \decomp_{i_2}$. 
    Then, depending on $i_1$, either the fact that $|\ext(R_1)| = 1$ (for $i_1 = 0$) or the induction hypothesis (for $i_1 > 0$) together with Proposition~\ref{prop:monotonicity-of-hulls-wrt-parameters} imply that there exists a block $R_1' \in \repwconv<\theta_{i}>[E^+]$ such that $\ext(R_1) \subseteq \ext(R_1')$.
    Using a similar argument, there exists a block $R_2' \in \repwconv<\theta_{i}>[E^+]$ such that $\ext(R_2) \subseteq \ext(R_2')$.
    Furthermore, by the choice of $R_1,R_2$, there are $x \in \ext(R_1)$ and $y \in \ext(R_2)$ such that $D(x,y) \leq \theta_i$.
    Thus, the distance between $\ext(R_1')$ and $\ext(R_2')$ is at most $\theta_i$ and hence $R_1' = R_2'$ by Theorem~\ref{thm:weak-convex-decomposition}.
    The claim then follows from (\ref{eq:join_generators}), (\ref{eq:join}), and the monotonicity of $\wconv<\theta_i>$.
\end{proof}

Defining $\theta_0 = 0$, below we claim that $\decomp_i$ is a representation of $\wconv<\theta_i>(E^+)$ for all $i \geq 0$.
\begin{lemma} 
	\label{lm:IWCH}
    For all $i \geq 0$, $\decomp_i=\repwconv<\theta_i>[E^+]$. 
\end{lemma}
\begin{proof} 
    The statement is trivial for $i = 0$. 
	Regarding $i > 0$, we claim that the family of the extensions of the blocks in $\decomp_i$ satisfies conditions \ref{proof-enum:components-weakly-convex}--\ref{proof-enum:components-distance} of Theorem~\ref{thm:weak-convex-decomposition}. 
	Indeed, (\ref{eq:join}) and Proposition~\ref{pr:join} together imply \ref{proof-enum:components-weakly-convex} and \ref{proof-enum:components-theta-connected}, whereas \ref{proof-enum:components-distance} follows by construction (cf. line~\ref{line:next_theta}). Thus, $C = \bigcup_{R \in \decomp_i} \ext(R)$ and the family $(\ext(R))_{R\in \decomp_i}$ fulfills the conditions of Theorem~\ref{thm:weak-convex-decomposition}. Hence, $C$ is $\theta_i$-convex. 
	But then, since $E^+ \subseteq C$, $\wconv<\theta_i>(E^+) \subseteq C$ holds by the monotonicity and idempotency of $\wconv<\theta_i>$. Furthermore, we have $C \subseteq \wconv<\theta_i>(E^+)$ by Lemma~\ref{lm:join}. 
	Thus, $C = \wconv<\theta_i>(E^+)$, completing the proof.
\end{proof}

Finally, the algorithm checks whether each new block in $\decomp_i$ is consistent with the negative examples by calling function \fn{Membership} defined by
\begin{equation}
	\label{eq:membership}
	\fn{Membership}(e,R) = 
	\begin{cases}
		\symboltrue{}	& \text{if $e \in \ext(R)$}\\
		\symbolfalse{}  & \text{o/w}
	\end{cases}
\end{equation}
for all $e \in E^-$ and $R \in \decomp_i \setminus \decomp_{i-1}$.
Similarly to \fn{Distance} and \fn{Join}, the algorithmic realization of \fn{Membership} is assumed to operate directly on $R$. 

Using the above definitions and considerations, we are ready to state our main result of this section concerning the correctness and complexity of Algorithm~\ref{alg:intCHF}.
We use the following notation in the theorem: $\TS$, $\TD$, $\TJ$, and $\TM$ denote the time complexity of functions \fn{Singleton}, \fn{Distance}, \fn{Join}, and \fn{Membership}, respectively.
\begin{theorem}
\label{thm:intCHF}
Let $\M$ be a blockwise convex metric space for some $\tau \geq 0$ and $\mu$ a representation scheme for $\M$ and $\tau$.
Then Algorithm~\ref{alg:intCHF} solves Problem~\ref{problem:intCHF} for $\M$ correctly in time
\begin{equation}
\bigo(m_\oplus^2\log{m_\oplus}+m_\oplus\TS+m_\oplus^2\TD+m_\oplus\TJ+m_\oplus m_\ominus \TM)\enspace ,  
\label{eq:timecomplexityofalgintCHF}
\end{equation}
where $m_\oplus=\card{E^+}$ and $m_\ominus=\card{E^-}$.
\end{theorem}
\begin{proof}
It is easy to check that Algorithm~\ref{alg:intCHF} is correct if it returns ``{\sc NO}'' or $\decomp_0$. 
Otherwise, by \ref{lm:wellbehaved_and_blockwiseconvex:same-level-is-same} and \ref{lm:wellbehaved_and_blockwiseconvex:convex-is-last} of Lemma~\ref{lm:wellbehaved_and_blockwiseconvex}, it suffices to consider the $\theta_i$s in Lemma~\ref{lm:wellbehaved_and_blockwiseconvex} because those values already generate \textit{all} weakly convex hulls of $E^+$. 
Furthermore, \ref{lm:wellbehaved_and_blockwiseconvex:next-level-is-different} of Lemma~\ref{lm:wellbehaved_and_blockwiseconvex} guarantees that if the $\theta_i$-convex hull of $E^+$ is inconsistent with $E^-$ for some $i$, then all $\theta_j$-convex hulls of $E^+$ for $j \geq i$ are also inconsistent. 
These properties, together with Lemma~\ref{lm:IWCH}, imply the correctness of Algorithm~\ref{alg:intCHF}.

Regarding its time complexity, Algorithm~\ref{alg:intCHF} can be implemented by maintaining the sets
\begin{align*}
L_{BP} &=   \{(d,\{R_1,R_2\}): R_1,R_2\in \decomp \text{ with } 0 < \text{\fn{Distance}}(R_1,R_2) = d\} \\ 
\shortintertext{and} 
L(R) &= \{ (d,\{R_1,R_2\}) \in L_{BP} : R \in \{R_1,R_2\} \}  
\end{align*}
for all $R \in \decomp$.
$L_{BP}$ (resp. $L(R)$) is used to quickly find two blocks with distance at most a threshold (resp. the nodes of $L_{BP}$ that refer to $R$).
They can be realized by a red-black (RB) tree and by doubly linked lists, respectively.
Since insertion in RB trees (resp. in doubly linking lists) can be performed in logarithmic (resp. constant) time, the time complexity of the initialization (line~\ref{line:init}) is
\begin{equation}
	\label{eq:complexity_init_ICHF}
	\bigo(m_\oplus^2 \log{m_\oplus} + m_\oplus \TS +m_\oplus^2 \TD) \enspace .
\end{equation}
For the execution of line~\ref{line:join}, one can select an \textit{arbitrary} pair $R_1,R_2\in \decomp$ that have distance at most $\theta_i$ (e.g., the pair in the root of the RB tree) and proceed as follows:
	\begin{itemize}
		\item[($\alpha$)] 
        Delete all nodes $N$ of the RB tree that contain $R_1$ or $R_2$ in their second entry as well as all occurences of $N$ in all doubly linking lists.
		\item[($\beta$)] 
		Compute the new block $R$ by joining $R_1$ and $R_2$ and set $L(R) = \emptyset$. 
		\item[($\gamma$)] For all blocks $R' \in \decomp \setminus \{R_1,R_2,R\}$, compute the distance $d$ between $R$ and $R'$, insert $N=(d,\{R,R'\})$ into the RB tree, and add $N$ to $L(R)$ and $L(R')$.
	\end{itemize}
Suppose $|\decomp| = m$ before the execution of line~\ref{line:join}. 
The algorithm carries out $\bigo(m)$ deletions in the RB tree and $\bigo(m)$ deletions in the doubly linked lists for ($\alpha$), 
one join operation for ($\beta$), $\bigo(m)$ distance calculations and $\bigo(m)$ insertions for ($\gamma$). 
Line~\ref{line:join} is carried out at most $m_\oplus-1$ times because $|\decomp| = m_\oplus$ initially and each execution of line~\ref{line:join} decreases the cardinality of $\decomp$ by one.  
Together with (\ref{eq:complexity_init_ICHF}), this implies (\ref{eq:timecomplexityofalgintCHF}) in the claim, by noting that the algorithm spends $\bigo(m_\oplus m_\ominus \TM)$ total time for checking the consistency in line~\ref{line:start_consistency} and that the insertion and deletion operation in RB  trees (resp. doubly linking lists) can be carried out in logarithmic (resp. constant) time.
\end{proof}

\subsection{Some Illustrative Examples}
\label{sec:k-convexity}
In this section, we present some examples to illustrate the application of Algorithm~\ref{alg:intCHF} and Theorem~\ref{thm:intCHF} for three different domains.

\subsubsection*{Learning Weakly Convex Boolean Functions}

As a first application of Theorem~\ref{thm:intCHF}, we prove that the CHF problem for \emph{weakly convex} Boolean functions can be solved in {polynomial} time.
This result is not new, it was shown with a \emph{domain-specific} CHF algorithm in \citep{Ekin2000}.
Nevertheless, we present this application example because it clearly demonstrates some nice properties of our \emph{general purpose} algorithm. 
In particular, Algorithm~\ref{alg:intCHF} solves the related CHF problem  in the \textit{same} asymptotic time complexity as the domain-specific algorithm by \citet{Ekin2000} and this positive result can be obtained in a fairly \textit{simple} way by  using Algorithm~\ref{alg:intCHF} and Theorem~\ref{thm:intCHF}.


More precisely, we consider the Hamming metric space $\MHamming = (\Hamming_n,\dHamming)$ for some $n \in \N$ (cf.
Section~\ref{sec:preliminaries}).
A Boolean function $f: \Hamming_n \to \{0,1\}$ is \emph{$\theta$-convex} for some $\theta \geq 0$ if its extension
$\ext(f) = \{x \in \Hamming_n: f(x)=1\}$ is  $\theta$-convex in $\MHamming$.
In order to define a suitable representation scheme for $\theta$-convex Boolean functions, we need some further
notions.
The set $\{x_1,\neg x_1,\ldots,x_n,\neg x_n\}$ of Boolean \textit{literals} is denoted by $L_n$.
A \emph{term} $T$ is a conjunction of literals from $L_n$; $T$ is sometimes regarded as the set of literals it
contains.
A \emph{conflict} between two terms $T_i$ and $T_j$ over $L_n$ is an integer $p\in [n]$ such that $x_p \in T_i$ and
$\neg x_p \in T_j$ or vice versa.
We first claim some important properties of $\MHamming$.

\begin{proposition}
	\label{pr:BooleanProperties}
	For all $n \geq 0$, $\MHamming=(\Hamming_n,\dHamming)$ is blockwise convex for $\tau = 2$. In particular, for all $A \subseteq\Hamming_n$ that are $2$-convex and $2$-connected, $A$ is a Boolean subcube of $H_n$.
\end{proposition}
\begin{proof}
    Since $\MHamming$ is finite, it is complete. We prove for $A$ in the claim that $A = \Hamming_n[A]$, where $\Hamming_n[A]$ is the
	\textit{smallest} Boolean subcube of $\Hamming_n$ containing $A$.
	By definition, $A \subseteq \Hamming_n[A]$.
	To show $A \supseteq \Hamming_n[A]$, note that the conditions on $A$ imply that $A$ is connected (i.e., $1$-connected), from which we
	have $A \supseteq \Hamming_n[A]$ by the result that a connected Boolean function is convex if and only if it is
	$2$-convex (cf.
	Theorem 5.16 in \cite{Ekin1999}) and by the fact that a subset of $\Hamming_n$ is convex if and only if it is a subcube
	of $\Hamming_n$.
\end{proof}

We have the following result for the CHF problem for weakly convex Boolean functions: 
\begin{theorem}
\label{thm:intCHFBoolean} 
    For all $n \geq 0$, there is a representation scheme $\mu$ for
    $\MHamming=(\Hamming_n,\dHamming)$ and $\tau =2$.
	Furthermore, Algorithm~\ref{alg:intCHF} solves Problem~\ref{problem:intCHF} for $\MHamming$, $\mu$, and $\tau$ in time
    \begin{equation}
        \label{eq:intCHFBoolean:complexity} 
        \bigo(n m_{\oplus} (m_{\oplus} + m_{\ominus})) \enspace .
    \end{equation}
\end{theorem}
\begin{proof}
    Let $\theta \geq 2$.
	By  Proposition~\ref{pr:BooleanProperties} we have that $\MHamming$ is blockwise convex for $\theta$ and that the blocks of $\theta$-convex sets are formed by (Boolean) subcubes of $\Hamming_n$.
    Utilizing the fact that any non-empty subcube of $\Hamming_n$ can \textit{uniquely} be represented by a term over $L_n$, we define $\blockrepwconvHamming(\theta,A)$ for all subsets $A \subseteq \Hamming_n$ by the term representing $\wconv(A)$, if $\wconv(A)$ is a non-empty subcube of $\Hamming_n$; otherwise by $\bot$.
    One can easily check that $\blockrepwconvHamming$ satisfies
	(\ref{enum:singleton-block-representations}) and (\ref{enum:block-representations-coherent}).
    Theorem~\ref{thm:weak-convex-decomposition} then implies that $\repwconvHamming: \R_{\geq 2} \times 2^{\Hamming_n} \to 2^{\{0,1\}^*}$ defined by 
    \begin{equation}
    \label{eq:representation_Boolean_functions}
		\repwconvHamming(\theta, A) = \{\blockrepwconvHamming(\theta, B) : B \in \blocks(\wconv(A))\}
    \end{equation}
    for all $A \subseteq \Hamming_n$ is a representation scheme for $\MHamming$ and $\tau=2$.
    Note that $\repwconvHamming(\theta, A)$ in (\ref{eq:representation_Boolean_functions}) is a \textit{$k$-term-DNF} with $k = \card{\blocks(\wconv(A))}$.

    Defining $\mu$ in Problem~\ref{problem:intCHF} by $\mu_H$ in (\ref{eq:representation_Boolean_functions}), $\TS, \TD, \TJ, \TM$ in Theorem~\ref{thm:intCHF} are all in
	$\bigo(n)$  time for $\M=\MHamming$.
	This is trivial for $\TM$ and follows for $\TS$ directly from $\blockrepwconvHamming(\theta,\{x\}) =
		\bigwedge_{i=1}^n l_i$ with $l_i = x_i$, if $x_i = 1$; o/w  $l_i = \neg x_i$, for all $x =(x_1,\ldots,x_n)\in \Hamming_n$.
	Regarding $\TD$ and $\TJ$, let $T_i$ and $T_j$ be terms over $L_n$.
	Then $\dHamming(\ext(T_i), \ext(T_j))$ is equal to the number of conflicts between $T_i$ and  $T_j$, and for all $\theta\geq 2$ and terms $T_i,T_j$ with $\dHamming(\ext(T_i), \ext(T_j)) \leq\theta$, $\Call{Join}{\theta,T_i,T_j}$ is the term with literals $T_i \cap T_j$.
    We get (\ref{eq:intCHFBoolean:complexity}) from the general bound (\ref{eq:timecomplexityofalgintCHF}) in Theorem~\ref{thm:intCHF} by noting that $\log{m_\oplus} = \bigo(n)$. 
\end{proof}
A few remarks are in order.
First, unless R $=$ NP, it is NP-hard to find a consistent $k$-term-DNF, i.e., $k$ subcubes of $\Hamming_n$ such that their union is consistent with the examples,  for the \textit{smallest} $k$~\citep{Pitt1988}.  
While there is no restriction on the subcubes in this problem, weakly convex Boolean functions require a minimum distance between them. 
Although it is not guaranteed that the weakly convex hull returned by Algorithm~\ref{alg:intCHF} is optimal with respect to the number of blocks among \textit{all} consistent weakly convex hypotheses, it is an \textit{efficiently} computable alternative to the computationally \textit{intractable} smallest consistent $k$-term-DNF.
Second, the time complexity of the related domain-specific algorithm in \citep{Ekin2000} is 
slower by a factor of $\log n$.
However, that factor can be saved by applying the idea in our Algorithm~\ref{alg:intCHF} that linear search enables for an incremental calculation of the $\theta$-convex hulls for increasing $\theta$s.
This is faster than binary search computing them \emph{from scratch}.
Third, \cite{Ekin2000} also prove that for all $\theta > n/2 - 1$, the concept class $\ccWConvBoolean$ is \textit{polynomially} PAC-learnable, where $\ccWConvBoolean$ is defined as follows: For all $A \subseteq \Hamming_n$, $A \in \ccWConvBoolean$ if and only if $A$ is $\theta$-convex. 
Their proof is based on showing that the CHF problem can be solved in polynomial time for $\ccWConvBoolean$ (see, also, \ref{thm:BEHW:PAC-if-VC-and-CHF} of Theorem~\ref{thm:BEHW}).
This is trivial for $n < 6$; for $n \geq 6$, it can be shown by Theorem~\ref{thm:intCHFBoolean}, a special case of Theorem~\ref{thm:intCHF}, in a fairly \textit{simple} way.
Indeed, since $\theta > 2$, Theorem~\ref{thm:intCHFBoolean} guarantees that there exists a consistent hypothesis in $\ccWConvBoolean<\theta>$ if and only if Algorithm~\ref{alg:intCHF} with $\MHamming$, $\mu$ in (\ref{eq:representation_Boolean_functions}) and for $E^+,E^- \subseteq \Hamming_n$ returns a solution $\decomp$ of the CHF problem in polynomial time such that for all $R_1,R_2 \in \decomp$ with $R_1 \neq R_2$, $\dHamming(\ext(R_1),\ext(R_2)) > \theta$.

\subsubsection*{Learning Weakly Convex Unions of Axis-Aligned Hyperrectangles}
As a second illustrative example for the application of Algorithm~\ref{alg:intCHF} and Theorem~\ref{thm:intCHF}, we show that the CHF problem can be solved for \textit{weakly convex} unions of axis-aligned hyperrectangles in polynomial time. 
The underlying metric space for this example is $\MUnitCube = (\unitdcube,D_1)$, where $\unitdcube = [0,1]^d$ denotes the \emph{unit $d$-cube}. 
Note that $\MUnitCube$ can be regarded as a generalization of $\MHamming$ considered above.
Before stating our result in Theorem~\ref{thm:intCHFRectangles}, we first formulate some basic properties of $\MUnitCube$.

\begin{proposition}
	\label{pr:UnitdCube}
	For all $d \geq 0$ integer, $\MUnitCube$ satisfies the following properties:
	\begin{enumerate}[label=(\roman*)] 
        \item It is blockwise convex for any $\tau > 0$.
		      In particular, for all $A \in \finitesubsets[\unitdcube]$ such that $\wconv<\tau>(A)$ is $\tau$-connected,
		      $\wconv<\tau>(A)=\unitdcube{[A]}$, where $\unitdcube{[A]}$ is the smallest axis-aligned subcube that
		      contains $A$.
		      \label{pr:UnitdCube:blockwise-convex}
        \item 
            For all $A \in \finitesubsets[\unitdcube]$ and $\theta \geq 0$, $\wconv(A)$ is a finite union of axis-aligned closed hyperrectangles.
            \label{pr:unitdCube:blocks}
	\end{enumerate}
\end{proposition}
\begin{proof}
    Regarding \ref{pr:UnitdCube:blockwise-convex}, the completeness holds by the definition of $\unitdcube$.
	Let $A$ be a subset of $\unitdcube$ satisfying the conditions in \ref{pr:UnitdCube:blockwise-convex}.
	One can easily check that if $A$ is not $\tau$-connected then there exists a $\tau$-connected set
	$B\in\finitesubsets[\unitdcube]$ such that $A \subset B$ and $\wconv<\tau>(B) = \wconv<\tau>(A)$.
	Thus, it suffices to consider the case that $A$ is $\tau$-connected.
	We prove \ref{pr:UnitdCube:blockwise-convex} by induction on $\card{A}$.
	The base case $\card{A}=1$ is trivial.
	Suppose the claim holds for all $\tau$-connected sets $A' \subset \unitdcube$ with $\card{A'}\leq k$.
	Let $A =A' \cup\{a\}$ for some $A' \in \finitesubsets[\unitdcube]$ and $a \in \unitdcube$ such that $\card{A'}=k$ and $A,A'$ are both $\tau$-connected.
	The claim holds directly by the induction hypothesis if $\wconv<\tau>(A)=\wconv<\tau>(A')$.
	Suppose $a \notin \wconv<\tau>(A')$.
	Clearly, $\wconv<\tau>(A)\subseteq \unitdcube{[A]}$.
	Conversely, let $x =(x_1,\ldots,x_d) \in \unitdcube{[A]}$,  $\text{\sc Min}_i=\min_{y\in A'}y[i]$, and $\text{\sc
			Max}_i=\max_{y\in A'}y[i]$ for all $i\in[d]$.
	Let $a'=(a_1',\ldots,a_d') \in \unitdcube{[A']}$ be the point with the smallest $D_1$ distance to $a$.
	We show that $(x_1,a_2,\ldots,a_n)\in \wconv<\tau>(A)$.
	The claim is automatic if $x_1 = a_1$.
	Otherwise, $a' \in \wconv<\tau>(A')$ holds by the induction hypothesis and hence, $D_1(a,a') \leq \tau$.
	Thus, for $\tau_1 = \abs{a_1-a_1'}$ we have $\tau_1 \leq \tau$.
	We prove the claim only for the case that $a_1 < a_1'$; the proof for $a_1 \geq a_1'$ can be shown with similar
	arguments.
	If $x_1 \in \lbrack a_1, a_1'\rbrack$, then $(x_1,a_2,\ldots,a_n)\in \wconv<\tau>(\{a,a'\})\subseteq\wconv<\tau>(A)$.
 
	Otherwise (i.e., $x_1 \in (a_1', \text{\sc Max}_1]$), let $p_i$ (resp.
	$p'_j$) denote $(a_1+i\tau_1,a_2,\ldots,a_d)\in \unitdcube{[A]}$ (resp.
	$(a_1'+j\tau_1,a_2',\ldots,a_d')\in\unitdcube{[A']}$) for all $i=0,\ldots,\ell$ (resp. $j=0,\ldots,\ell-1$), where
	$\ell = \lfloor (x_1-a_1)/\tau_1\rfloor$.
		By the induction hypothesis, $p_j'\in \wconv<\tau>(A')\subseteq \wconv<\tau>(A)$ for all $j$ and hence, $p_i \in
	\wconv<\tau>(\{p_{i-1},p_{i-1}'\})\subseteq\wconv<\tau>(A)$ holds for all $i\in [\ell]$.
		Therefore, $$(x_1,a_2,\ldots,a_n)\in  \wconv<\tau>(\{p_\ell,(\min\{a_1'+\ell\tau_1,\text{\sc
				Max}_1\},a_2',\ldots,a_d')\})\in \wconv<\tau>(A)\enspace .
		$$
		Using the same arguments, we have 
		$$
			(x_1,\ldots,x_{i},a_{i+1},\ldots,a_d) \in \wconv<\tau>(A'\cup\{(x_1,\ldots,x_{i-1},a_{i},\ldots,a_d)\}) \subseteq
			\wconv<\tau>(A)
		$$
		for $i=2,\ldots,d$, completing the proof of \ref{pr:UnitdCube:blockwise-convex}.

        Regarding \ref{pr:unitdCube:blocks}, $\wconv(A) = A$ for $\theta = 0$ and each block consists of a single point. It is closed as $A$ is finite. 		For $\theta > 0$ the claim follows from Theorem~\ref{thm:weak-convex-decomposition} and \ref{pr:UnitdCube:blockwise-convex}.
\end{proof}
We are ready to state the following result for Problem~\ref{problem:intCHF} for $\M =\MUnitCube$: 
\begin{theorem}
	\label{thm:intCHFRectangles} 
    For all $d \in \N$ and $\tau > 0$, there is a representation scheme
	$\repwconvUnitCube$ for $\MUnitCube$ and $\tau$.
	Furthermore, Algorithm~\ref{alg:intCHF} solves Problem~\ref{problem:intCHF} for $\MUnitCube$, $\repwconvUnitCube$, and
	$\tau$ in time
    \begin{equation}
    \label{eq:complexity_hyperrectangles}
    \bigo(m_{\oplus}^2 \log{m_{\oplus}} + m_{\oplus}^2 d + m_{\oplus} m_{\ominus} d) \enspace. 
    \end{equation}
\end{theorem}
\begin{proof}
    Let $\theta > 0$.
	By Proposition~\ref{pr:UnitdCube}, $\MUnitCube$ is blockwise convex for $\theta$.
	Furthermore, by \ref{pr:unitdCube:blocks} of Proposition~\ref{pr:UnitdCube}, for all	$d \geq 0$ and $A \in \finitesubsets[\unitdcube]$, $\wconv(A)$ is the union of $k$ axis-aligned hyperrectangles of $\unitdcube$, where $k = \card{\blocks(\wconv(A))}\leq \card{A}$. 
    Utilizing the fact that an axis-aligned hyperrectangle can be represented by its minimum and maximum vertices, we define $\blockrepwconvUnitCube(\theta,A)$ for all $A \subseteq \unitdcube$ by $(A_{\min},A_{\max})$ if $A = \unitdcube{[A]}$; otherwise by $\bot$, where $A_{\min}$ (resp. $A_{\max}$) denotes the componentwise minimum (resp. maximum) of the points in $A$.\footnote{We assume that real numbers are represented in $\bigo(1)$ space up to a
		certain precision.}
    Clearly, $\blockrepwconvUnitCube$ satisfies (\ref{enum:singleton-block-representations}) and (\ref{enum:block-representations-coherent}).
	Define $\repwconvUnitCube:\R_{> 0} \times \finitesubsets[\unitdcube] \to 2^{\{0,1\}^*}$ by $$\repwconvUnitCube(\theta,A) =
		\{\blockrepwconvUnitCube(B) : B \in \blocks(\wconv(A))\}$$ 
    for all  $d \geq 0$ and $A \in \finitesubsets[\unitdcube]$.
	It holds that $\repwconvUnitCube$ is a representation scheme for $\MUnitCube$ and $\tau$.
	Defining $\mu$ in Problem~\ref{problem:intCHF} by $\repwconvUnitCube$,  (\ref{eq:complexity_hyperrectangles}) then follows by
	Theorem~\ref{thm:intCHF} by noting that $\TS, \TD, \TJ, \TM$ in Theorem~\ref{thm:intCHF} are all in $\bigo(d)$ time for
	$\M=\MUnitCube$.
\end{proof}

A few comments on this result are in order.
First, similar to Theorem~\ref{thm:intCHFBoolean}, we obtained Theorem~\ref{thm:intCHFRectangles} in a fairly \textit{simple} way by using our general result in Theorem~\ref{thm:intCHF}.
Second, the number of hyperrectangles returned by Algorithm~\ref{alg:intCHF} in polynomial time is \textit{optimal} with respect to the set of weakly convex hulls of the positive examples. 
However, the hyperrectangles must be pairwise \textit{non-overlapping}.
In contrast, it is NP-hard to find the smallest number of possibly \textit{overlapping} axis-aligned rectangles whose union is consistent with the examples, even for $d=2$~\citep{Bereg2012}. 
Third, as we show below in Theorem~\ref{thm:PACrecangles} by using Theorem~\ref{thm:intCHFRectangles}, the concept class formed by the $\theta$-convex union of axis-aligned hyperrectangles is polynomially PAC-learnable.
More precisely, we prove that for all $d > 0$, the concept class $\ccWConvRectangles = \{ \wconv(A) : A \in \finitesubsets[\unitdcube] \}$ is polynomially PAC-learnable for sufficiently large $\theta$.
We prove the claim for $d > 0$ by noting that it is straightforward for $d = 0$.

\begin{theorem} 
\label{thm:PACrecangles} 
    For any constant $c\geq 0$, $\ccWConvRectangles$ is polynomially PAC-learnable for all $d \in \N$ and $\theta \geq \frac{2d}{\euler} \sqrt[d]{\frac{\euler}{d^{c-1}}}$.
\end{theorem}
\begin{proof}
	If $\theta \geq d$, then Proposition~\ref{pr:UnitdCube} implies that all concepts in $\ccWConvRectangles$ consist of a single axis-aligned hyperrectangle of $\unitdcube$; this concept class is known to be efficiently PAC-learnable. 
	For the case that $\frac{2d}{e} \sqrt[d]{\frac{e}{d^{c-1}}} \leq \theta < d$, by \ref{thm:BEHW:PAC-if-VC-and-CHF} of Theorem~\ref{thm:BEHW} it is sufficient to show that the CHF problem for $\ccWConvRectangles$ can be solved in polynomial time and that the VC-dimension of $\ccWConvRectangles$ is bounded by a polynomial of $d$. 
    Setting $\tau = \frac{2d}{e} \sqrt[d]{\frac{e}{d^{c-1}}}$, we have $\theta > \tau = 0$ by $d > 0$.
    Proposition~\ref{pr:UnitdCube} and Theorem~\ref{thm:intCHFRectangles} imply that there exists a consistent hypothesis in $\ccWConvRectangles$ if and only if Algorithm~\ref{alg:intCHF} with $\MUnitCube$, $\tau$, and $\mu$ defined above and for input $E^+,E^- \in \finitesubsets[\unitdcube]$ returns a solution $\decomp$ of the CHF problem such that for all $R_1,R_2 \in \decomp$ with $R_1 \neq R_2$, $D_1(R_1,R_2)  > \theta$. 
    Furthermore, Algorithm~\ref{alg:intCHF} runs in polynomial time.
    The proof then follows by Lemma~\ref{lm:VC-dimension_hyperrectangles} (see the Appendix), which states that
    \begin{equation*}
        \vcdim(\ccWConvRectangles)= \bigo(d^{c+1} \log d) \enspace .
    \end{equation*}
\end{proof}

\subsubsection*{Learning Weakly Convex Unions of Polygons}

In the previous example, the Manhattan distance $D_1$ induced axis-aligned hyperrectangles over $\Rd$.
Our third example is concerned with the metric space $\MPlane = (\R^2, \dist_2)$.
For this case, the Euclidean distance $D_2$ induces weakly convex unions of convex polygons.
Using Algorithm~\ref{alg:intCHF} and Theorem~\ref{thm:intCHF}, we show that the CHF problem can also be solved efficiently for this class of weakly convex hypotheses. 

To present this result, we first recall some necessary notions.
A point $p$ of a convex set $C \subseteq \R^2$ is \textit{extreme} if there are no  $x,y \in C$ such that $x, y , p$ are pairwise different and $p \in \ti(x, y)$. 
Let $\extreme(C)$ denote the set of all extreme points of $C$. 
It is a well-known fact that if $A \in \finitesubsets[\R^2]$, then $\conv(A)$ is a convex polygon and $\extreme(\conv(A)) \subseteq A$ \citep[see, e.g.,][]{Krein1940}. 
A subset $A \subseteq \R^2$ is called \textit{path-connected} if for all $x, y \in A$ there is a continuous function $f: [0, 1] \to A$ with $f(0) = a$ and $f(1) = b$. 
Moreover, $A$ is called \textit{locally convex} if for every $x \in A$ there exists $\delta_x > 0$ such that $\oball<\delta_x>(x) \cap A$ is convex, where
$\oball(x) = \{y \in X: \dist(x, y) < r\}$. 

\begin{proposition}
    \label{prop:polygons} 
    $\MPlane$ satisfies the following properties: 
    \begin{enumerate}[label=(\roman*)]
        \item \label{prop:polygons:blockwise} For all $\tau > 0$, $\MPlane$ is blockwise convex for $\tau$. In particular, if $\wconv<\tau>(A)$ is $\tau$-connected for some $A \in \finitesubsets[\R^2]$ then $\wconv<\tau>(A) = \conv(A)$ is a convex polygon with $\extreme(\conv(A)) \subseteq A$. 
        \item \label{prop:polygons:wconv-decomposition} For all $A \in \finitesubsets[\R^2]$ and $\theta \geq 0$, $\wconv(A)$ is the union of a finite set of convex polygons. 
    \end{enumerate}
\end{proposition}
\begin{proof}
    Regarding \ref{prop:polygons:blockwise}, clearly, $\MPlane$ is complete. Let $\tau > 0$ and $A \in \finitesubsets[\R^2]$ such that $\wconv<\tau>(A)$ is $\tau$-connected.
    We first show that $\wconv<\tau>(A)$ is (a) closed, (b) path-connected, and (c) locally convex. 
    Property~(a) is immediate from the definition of  weakly convex sets. 
    Since $\wconv<\tau>(A)$ is $\tau$-connected and $\tau$-convex, 
    for any two of its points there is a polygonal chain and hence, a continuous path in $\wconv<\tau>(A)$ connecting them, implying (b).
    To show (c), let $x \in \wconv<\tau>(A)$, $\delta_x = \tau/2$, and $y, z \in \oball<\delta_x>(x) \cap \wconv<\tau>(A)$.
    Then  $\euclideandist(y, z) \leq \tau$ and hence, $\ti(y, z) \subseteq \wconv<\tau>(A) \cap \oball<\delta_x>(x)$ because $\wconv<\tau>(A)$ is $\tau$-closed and $\oball<\delta_x>(x)$ is convex, implying (c).
    Applying the result shown independently by \citet{Tietze1928} and \citet{Nakajima1928} to $\MPlane$, 
    any closed, path-connected\footnote{
    In its general form, the Tietze-Nakajima theorem is stated for subsets of $\M_d = (\R^d, \euclideandist)$ that are closed, connected, and locally convex.
    A subset $A \subseteq \R^d$ is \textit{connected} if there are no open sets $A_1, A_2 \subseteq \R^d$ such that $A \cap A_1,A \cap A_2 \neq \emptyset$, $A_1 \cap A_2 \cap A = \emptyset$, and $A \subseteq A_1 \cup A_2$, where a set $B \subseteq \R^d$ is \textit{open} if for all $x \in B$ there is $\epsilon > 0$ such that $\oball<\epsilon>(x) = \{y \in \R^d : \euclideandist(x, y) < \epsilon\} \subseteq B$. It is a well-known fact that path-connectedness implies connectedness even in general topological spaces.}, and locally convex set in $\MPlane$ is convex. 
    Thus $\wconv<\tau>(A)$ is convex 
    and hence $\conv(A) \subseteq \wconv<\tau>(A)$, which, together with \ref{prop:monotonicity:monotone-wrt-theta} of Proposition~\ref{prop:monotonicity-of-hulls-wrt-parameters}, implies $\conv(A) = \wconv<\tau>(A)$.
    The proof of \ref{prop:polygons:blockwise} is then completed by noting that the convex hull of any finite point set $A$ in $\MPlane$ forms a convex polygon whose extreme points lie in $A$~\citep[see, e.g., ][]{Krein1940}.  
    Finally, the proof of \ref{prop:polygons:wconv-decomposition} is immediate by \ref{prop:polygons:blockwise} and Theorem~\ref{thm:weak-convex-decomposition}. 
\end{proof}

We are ready to state the following result for Problem~\ref{problem:intCHF} for the  case of $\M = \MPlane$. 

\begin{theorem}
    \label{thm:intCHF:polygons} 
    For all $\tau > 0$, there is a representation scheme $\repwconvPlane$ for $\MPlane = (\R^2, \euclideandist)$ and $\tau$. Furthermore, Algorithm~\ref{alg:intCHF} solves Problem~\ref{problem:intCHF} for $\MPlane$, $\repwconvPlane$, and $\tau$ in time 
    \begin{equation}
    \label{eq:time_complexity_polygons}
        \bigo( m_\oplus^2 \log{m_\oplus} + m_\oplus m_\ominus \log{m_\oplus} ) \enspace \text{.}
    \end{equation}
\end{theorem}
\begin{proof}
    By Proposition~\ref{prop:polygons}, $\MPlane$ is complete and blockwise convex for any $\tau > 0$.
    Furthermore, \ref{prop:polygons:wconv-decomposition} of Proposition~\ref{prop:polygons} implies that for all $A \in \finitesubsets[\R^2]$, $\wconv(A)$ is the union of $k$ convex polygons where $k = \card{\blocks(\wconv(A))} \leq \card{A}$.
    Define $\repwconvPlane: \R_{>0} \times \finitesubsets[\R^2] \to \{0, 1\}^*$ by 
    $$\repwconvPlane(\theta, A) = \{\ccextreme(B) : B \in \blocks(\wconv(A))\}$$ 
    for all $A \in \finitesubsets[\R^2]$, where $\ccextreme(B)$ is the sequence of the extreme points of the convex polygon $B$ in counterclockwise order, starting with some canonical (e.g., the lexicographically smallest) extreme point. 
    One can easily check that $\repwconvPlane$ is a representation scheme for $\MPlane$ and $\tau$ with $\blockrepwconvPlane(\theta, A)$ defined by $\ccextreme(\wconv(A))$ for all $A \in \finitesubsets[\R^2]$ if $\wconv(A)$ is a convex polygon in $\MPlane$; otherwise by $\bot$. 
    Defining $\mu$ in Problem~\ref{problem:intCHF} by $\repwconvPlane$, for $\TS$, $\TD$, $\TJ$, $\TM$ in Theorem~\ref{thm:intCHF} we have that $\TS$ can be carried out in $\bigo(1)$ time by noting $\blockrepwconvPlane(\theta,\{x\}) = (x)$, %
    $\TD$ in $\bigo(\log{m_{\oplus}})$ \citep{Edelsbrunner1985},  %
    $\TJ$ in $\bigo(m_\oplus \log{m_\oplus})$ using, e.g., Graham's scan \citep{Graham1972} for computing the convex hull of the extreme points of two blocks, and $\TM$ in $\bigo(\log{m_\oplus})$ time by partitioning the plane into $m_\oplus$ wedges~\citep[see Chapter~2 of][]{Preparata/Shamos/85}. 
    We obtain (\ref{eq:time_complexity_polygons}) by substituting these time complexities into (\ref{eq:timecomplexityofalgintCHF}). 
\end{proof}

A few remarks on the algorithmic details are in order. 
All algorithms discussed in the proof of Theorem~\ref{thm:intCHF:polygons} work with a sequence of extreme points in counterclockwise order as a representation for the involved convex polygons. 
The algorithm of \citet{Edelsbrunner1985} for computing the minimum distance of two polygons with $m_1$ and $m_2$ extreme points, respectively, has an asymptotic runtime of $\bigo(\log{m_1} + \log{m_2})$. 
It assumes, however, that the two convex polygons do not intersect. 
In our setting, it is easy to find examples where this does not hold during the execution of Algorithm~\ref{alg:intCHF}. 
This is not a problem because, as \citet{Edelsbrunner1985} also mentions, there are algorithms detecting the intersection of two convex polygons having the same time complexity \citep[see, e.g., ][]{Dobkin1983}. 
Last but not least, the algorithm in \citep{Preparata/Shamos/85} for deciding the membership problem in convex $m$-gons requires a $\bigo(m)$ time preprocessing step. It can be carried out directly after the join operation. 

Finally, since the VC-dimension of convex polygons is unbounded, we cannot apply Theorem~\ref{thm:BEHW} to prove polynomial PAC-learnability for the concept class formed by weakly convex unions of polygons.

\section{The Extensional Learning Setting}
\label{sec:extensional-setting}
In this section, we present Algorithm~\ref{alg:extCHF}, an adaptation of Algorithm~\ref{alg:intCHF} to the case of learning \textit{extensional} weakly convex hypotheses, i.e., which are given by enumerating their elements.
Accordingly, the domains are restricted to \emph{finite} metric spaces. 
This learning setting naturally arise when weakly convex sets have no concise representation, e.g., in case of performing vertex classification in graphs. 
Similarly to Problem~\ref{problem:intCHF} in Section~\ref{sec:intensional-setting}, we consider the case that $\theta$ is \textit{not} given in advance and return the \textit{largest} weakly convex hull of the positive examples that is consistent with the negative examples. 
Out of the consistent weakly convex hulls, it is the closest approximation of the convex hull of the positive examples.
More precisely, we consider the following CHF problem:
\begin{problem}
	\label{problem:extCHF}
	\emph{Given} a metric space $\M=(X,\dist)$ with $\card{X}=n$ for some positive integer $n$ and disjoint sets $E^+,E^- \subseteq X$ of positive and negative examples, \emph{return} 
	$$
	\max_{\theta \geq 0} \{\wconv(E^+): \wconv(E^+) \cap E^- = \emptyset \} \enspace .
	$$
\end{problem}
\begin{algorithm}[t]
	\begin{algorithmic}[1]
		\Require finite metric space $\M= (X,\dist)$
		\Input disjoint sets $E^+,E^- \subseteq X$ with $E^+ \neq \emptyset$	
		\Output $\wconv(E^+)$ such that $\wconv(E^+)\cap E^- = \emptyset$ and $\wconv<\theta'>(E^+)\cap E^- \neq \emptyset$ for all $\theta' > \theta$ satisfying $\wconv<\theta'>(E^+) \supsetneq \wconv(E^+)$
		\Statex 
        \State compute and sort all pairwise distances in $X$
        \label{line:ext:pairwise_distances}
		\State $C_0 \gets E^+$ 
		\For{$i = 1, \ldots,k$}
		\State $C \gets C_{i-1}$ 
        \label{line:ext_init_C}
		\While{$\exists x,y \in C$ with  $0 < \dist(x,y) \leq \theta_{i}$}
        \label{line:ext_start_while}
        \If{$\ti(x,y) \cap E^- = \emptyset$} $C \gets C \cup \ti(x,y)$ \hfill // cf. (\ref{def:weak-convexity}) for the def. of $\ti(x,y)$
        \label{line:ext_add_interval}
		\Else \	\Return $C_{i-1}$
		\EndIf		
        \EndWhile
        \State $C_i \gets C$
        \label{line:ext_C_i}
		\EndFor
		\State \Return $C_i$ 
	\end{algorithmic}
	\caption{\sc Extensional Consistent Hypothesis Finding} 
	\label{alg:extCHF}
\end{algorithm}
Note that $E^+ \cap E^- = \emptyset$ and $\conv_{0}(E^+) = E^+$ together imply that Problem~\ref{problem:extCHF} always has a solution. 
Let $\theta_1 < \ldots < \theta_k$ be the pairwise distances in $X$ computed in line~\ref{line:ext:pairwise_distances} of Algorithm~\ref{alg:extCHF}  and define $\theta_0$ by $0$.
The solution of Problem~\ref{problem:extCHF} can be obtained for some $\theta \in \{\theta_0,\theta_1,\ldots,\theta_k\}$ because for all $A \subseteq X$, $\conv_{\theta_i}(A) = \conv_{\theta}(A)$ for all $\theta \in [\theta_i,\theta_{i+1})$, for every $i =0,\ldots,k-1$, and $\conv_{\theta'}(A)=\conv(A)$ for all $\theta' \geq \theta_k$. 
Algorithm~\ref{alg:extCHF} utilizes this fact and the monotonicity stated in Proposition~\ref{prop:monotonicity-of-hulls-wrt-parameters}.
In particular, in iteration $i$ of the outer loop, it calculates $C_i$ from $C_{i-1}$ by setting $C$ to $C_{i-1}$ (line~\ref{line:ext_init_C}) and adding the interval of $x,y$ to $C$ for all $x,y\in C$ with distance at most $\theta_i$ (cf. lines~\ref{line:ext_start_while}--\ref{line:ext_add_interval}). 
It returns $C_i$ for the largest $i$ that is consistent with $E^-$.
One can easily check that $C_i = \wconv<\theta_i>(E^+)$  for all $i \geq 0$, implying the correctness of Algorithm~\ref{alg:extCHF}.

Regarding the time complexity of Algorithm~\ref{alg:extCHF}, one can maintain the set of pairs $x,y$ considered in line~\ref{line:ext_start_while} in an RB tree with their distances as keys. 
Utilizing the fact that each pair of points in $X$ is considered at most once, we need $\bigo(\log{n})$ time for the insertion and for the deletion of a pair in the RB tree, where $n = |X|$. 
Furthermore, the interval $\ti(x,y)$ can be calculated in $\bigo(n)$ time for all $x,y \in X$ from the pairwise distances computed in line~\ref{line:ext:pairwise_distances}.
Thus, since the total number of pairs $x,y$ considered in line~\ref{line:ext_start_while} is $\bigo(n^2)$, the total time of the outer loop is $\bigo(n^3)$. 
Denoting the time complexity of computing all pairwise distances in line~\ref{line:ext:pairwise_distances} by $T_P(\M)$, we have the following result: \begin{theorem}
	\label{thm:ECHF}
	Algorithm~\ref{alg:extCHF} solves Problem~\ref{problem:extCHF} correctly in
	$\bigo(T_P(\M)+n^3)$  time and $\bigo(n^2)$ space.
\end{theorem}

\subsection{Learning Weakly Convex Sets in Graphs}
\label{subsec:vertex-classification}  

In this section, we illustrate how Algorithm~\ref{alg:extCHF} works on vertex classification in graphs.
More precisely, we experimentally demonstrate on different synthetic graph datasets that already for a relatively small set of training examples, Algorithm~\ref{alg:extCHF} is able to return a hypothesis that closely approximates the unknown target concept.
For an undirected graph $G$, the underlying metric space is defined by $\M= (V(G), \dist_g)$, where $\dist_g$ is the \emph{geodesic} (or shortest-path) distance. 
For simplicity, $G$ is assumed to be connected. 
%


\begin{table}[t]
	\centering
	\begin{tabular}{lr@{\hspace{0.25cm}}r@{\hspace{0.25cm}}rr@{\hspace{0.25cm}}rr}
		\toprule
		& Number of &  Number of & \multicolumn{2}{c}{Diameter} & \multicolumn{2}{c}{Density} \\
		Type & \multicolumn{1}{c}{Vertices} & \multicolumn{1}{c}{Graphs} & \multicolumn{2}{c}{Mean $\pm$ SD}  & \multicolumn{2}{c}{Mean $\pm$ SD}  \\
		\midrule
		\multirow[t]{3}{*}{Grid} & 2500 & 845 & $100.77  \, \pm \hspace*{-1em}$ & $22.88$ & $1.037 \cdot 10^{-3} \, \pm \hspace*{-1em}$ & $16.404 \cdot 10^{-5}$ \\
		& 5625 & 594 & $137.40  \, \pm \hspace*{-1em}$ & $31.68$ & $0.495 \cdot 10^{-3}  \, \pm \hspace*{-1em}$ & $8.784 \cdot 10^{-5}$ \\
		& 10000 & 400 & $173.63  \, \pm \hspace*{-1em}$ & $41.01$ & $0.308 \cdot 10^{-3}  \, \pm \hspace*{-1em}$ & $5.722 \cdot 10^{-5}$ \\
		\multirow[t]{3}{*}{Delaunay} & 2500 & 130 & $56.10  \, \pm \hspace*{-1em}$ & $1.16$ & $2.274 \cdot 10^{-3}  \, \pm \hspace*{-1em}$ & $0.108 \cdot 10^{-5}$ \\
		& 5625 & 180 & $84.04  \, \pm \hspace*{-1em}$ & $1.32$ & $1.012 \cdot 10^{-3}  \, \pm \hspace*{-1em}$ & $0.023 \cdot 10^{-5}$ \\
		& 10000 & 178 & $112.45  \, \pm \hspace*{-1em}$ & $1.52$ & $0.570 \cdot 10^{-3}  \, \pm \hspace*{-1em}$ & $0.009 \cdot 10^{-5}$ \\
		\bottomrule
	\end{tabular}
	\caption{Details of the balanced classification graph datasets. 
 For each combination of graph type and size, the table shows the corresponding number of graphs and the mean and standard deviation of their diameters and edge densities. 
	}
	\label{tab:classification-dataset}
\end{table}

\subsubsection*{Datasets}
For the experiments we generated two types of \textit{synthetic} graph datasets: {complete} and {incomplete grids}, which are classical examples, e.g., in percolation theory \citep[see, e.g., ][]{Kesten1982}, as well as less regular graphs based on \emph{Delaunay} triangulations~\citep{Delaunay1934}. 
For each graph $G$, the target concept was defined by $\wconv(A)$ for some $A \subseteq V(G)$ selected at random and $\theta$ defined below.

All incomplete grids were generated from some complete {non-periodic} or {periodic} two-dimensional grid of size $\ell \times \ell$ by removing  $t\%$ of the edges, where a periodic grid is obtained from a non-periodic one by connecting its corresponding boundary vertices horizontally and vertically.
The edges for removal were selected uniformly at random, subject to the constraint that the resulting graph remained connected.
For all $\ell \in \{50, 75, 100\}$, $t \in \{0, 20, 40\}$, $\card{A} \in \{10, 20, 30, 40\}$, and $\theta \in \{10, 15, 20, 25, 30\}$, we generated $25$ periodic and $25$ non-periodic 
grids 
with target concept $V^+ = \wconv(A)$ and $V^- = V(G) \setminus V^+$.
Note that $t = 0$ corresponds to complete grids.
Our explorations clearly showed that extremely good classification results were obtained for strongly imbalanced graphs, either because $V^+$ formed an (almost) convex set (when $|V^+| \gg |V^-|$) or because most blocks of $V^+$ were singletons (when $|V^+| \ll |V^-|$).
For our learning experiments we therefore used only those graphs 
that satisfied the constraint $\card{V^+} / \card{V} \in [0.25, 0.75]$ (see Table~\ref{tab:classification-dataset} for more details). 

The synthetic graphs in the second type were generated by \emph{Delaunay} triangulations~\citep{Delaunay1934}.
For each graph $G$ we selected a finite set of points from the unit square $[0,1]^2$ uniformly at random for $V(G)$ and connected two points $u, v \in V(G)$ by an undirected edge if and only if they co-occur in a triangle of the Delaunay triangulation. 
To increase the diameter of the generated graphs, we only kept the 95\textsuperscript{th} percentile of edges with respect to the Euclidean length. 
To ensure connectivity, we deleted all isolated vertices after the edge removal.
In this way, we generated 25 Delaunay graphs, each with a target concept, for all combinations of $\card{V(G)} \in \{2500, 5625, 10000\}$, $\card{A} \in \{10, 20, 30, 40\}$, $\theta \in \{10, 15, 20, 25, 30\}$, and set $V^+ = \wconv(A)$ and $V^- = V(G) \setminus V^+$. 
For similar reasons discussed above for grids, we used only graphs satisfying $\card{V^+} / \card{V} \in [0.25, 0.75]$ for our experiments (see Table~\ref{tab:classification-dataset}). 

In the learning problem over a particular graph, the {target concept $V^+$} as well as $\theta$ are both {unknown} to the learning algorithm. 
About 64.4\% of the selected target concepts consist of a single block. 
However, this does not necessarily imply that these target concepts are convex (see, e.g., the example in Remark~\ref{rem:circle-not-blockwise-convex}). 
For each graph in the experiments, we constructed 10 learning tasks by sampling $\card{E^+} = \card{E^-} \in \{10, 20, 30, \dots, 100\}$ positive and negative examples independently and uniformly at random from $V^+$ and $V^-$, respectively.
We note that the hypotheses computed by Algorithm~\ref{alg:extCHF} can have \textit{two-sided} error. 
Besides false negatives, we can have also false positives. For example, the algorithm can overestimate $\theta$. 
Our early explorations of the experiments have, however, shown that this happened rarely.
Thus, we expect a precision 
of 1.0 in most of the cases. 
Accordingly, the {recall} 
is the interesting performance indicator. 
Nonetheless, we also measure the {accuracy} 
in order to compare our learner to the na\"{i}ve majority {baseline} classifier defined by $\max\{\card{V^+}/\card{V}, \card{V^-}/\card{V}\}$. 

\subsubsection*{Results}

\begin{figure}[t]
	\centering 
	\includegraphics{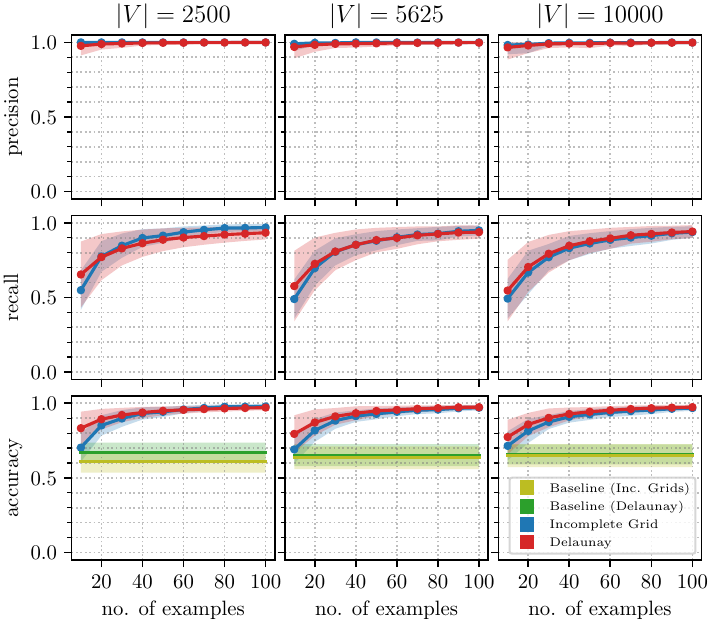}
	\caption{Precision, recall, and accuracy ($y$-axes) for various number of training examples ($x$-axes) for the balanced graphs with different graph sizes ($\card{V}$).}
	\label{fig:classification-results}
\end{figure}

The results are depicted in Figure~\ref{fig:classification-results}.
They are grouped vertically by the graphs' size (i.e., $\card{V}$). 
We plot the \emph{mean} precision, recall, and accuracy results ($y$-axes) obtained for different number of training examples ($x$-axes) with Algorithm~\ref{alg:extCHF} for grids (blue plots) and Delaunay graphs (red plots).\footnote{
	We note that in the case of Delaunay graphs, the experiments were carried out with weighted graphs as well, where the weight of an edge was defined by the Euclidean distance of its points. 
	The predictive performance in the weighted case was {slightly}, but not substantially worse.
	The overall picture was the same as presented in Figure~\ref{fig:classification-results}.}
In addition, we provide the mean baseline accuracy (green plots).
For all plots, the shaded area indicates one standard deviation from the mean value of the respective performance measure. 
For $\card{E^+} = \card{E^-} \geq 20$, our learner outperforms the baseline significantly. 
It is remarkable that the learner does \emph{not} require much more examples with increasing graph sizes to achieve the same performance. 
For example, on average $40$ training examples are sufficient to achieve an accuracy of at least 0.9, regardless of the graph type. 
One can also observe that for all graph types and graph sizes, the baseline is between 0.6 and 0.7 on average with a standard deviation of less than 0.078. 
This is due to the fact that it is defined by majority.

\begin{figure}[t]
	\centering
	\includegraphics{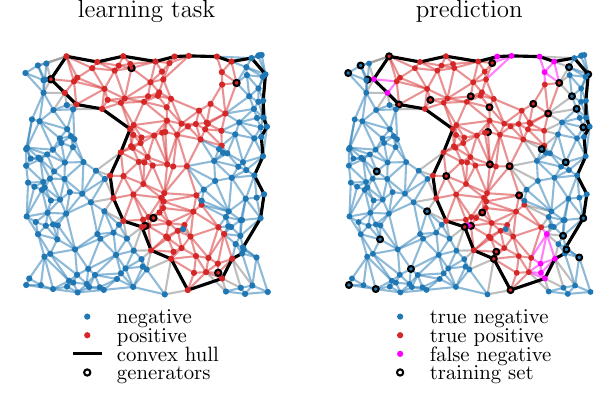}
	\caption{An exemplary Delaunay graph with $250$ vertices. 
 On the left, the unknown target concept (depicted in \textcolor{tabred}{\textbf{red}}). 
 It is the $\theta$-convex hull of the six generator vertices marked with black border for $\theta = 8$. 
 The target concept is \emph{not} convex; the convex hull of the generators contain the vertices enclosed by the black line.
 Notice that there is a negative point enclosed by three positive points in the lower part of the target concept. 
 On the right, the figure shows the prediction of the hypothesis returned by our generic algorithm for the $40$ training examples marked with a black border. 
 The image depicts \textcolor{tabred}{\textbf{true positives}}, \textcolor{tabblue}{\textbf{true negatives}}, and \textcolor{fuchsia}{\textbf{false negatives}}. 
 In this case, there were no false positives. 
 The convex hull of the positive examples contain the vertices enclosed by the black line. In this example it is the same as the convex hull on the left.
 }
	\label{fig:delaunay-task} 
\end{figure}


In Figure~\ref{fig:delaunay-task} we give an illustrative example of a learning task for a Delaunay-based graph with $\card{V} = 250$, together with the node prediction using $40$ trainings examples ($20$ positive and $20$ negative examples). 
The training examples are marked with black outline and the predictions are encoded by colors. 
In particular, dark red corresponds to  true positive, dark blue to true negative, and pink to false negative nodes. 
In this particular example we have no false positive node, which was the case for most graphs. 
Figure~\ref{subfig:illustrative-large-graph:medium-theta} in Section~\ref{sec:weakly-convex-properties} depicts one of the actual Delaunay target concepts that were used in our experiments for $10,000$ nodes. 
It consists of 3,518 nodes in 5 blocks. Notice the singleton block on the far bottom right.

In summary, our experimental results clearly show that using our \textit{generic} Algorithm~\ref{alg:extCHF}, a remarkable predictive accuracy can be obtained already with relatively small training sets, even though our approach does \textit{not} utilize any domain-specific knowledge. 
We emphasize that the focus of this paper is on investigating different aspects of the CHF problem for hypotheses over arbitrary, and not for some specific metric spaces.
The design and a systematic empirical evaluation of a domain-specific algorithm from our adaptation that, in addition, utilizes some structural properties of the underlying graph goes beyond the scope of this paper (cf. Section~\ref{sec:conclusion} for a discussion).


\section{Concluding Remarks}
\label{sec:conclusion} 
The illustrative examples in Sections~\ref{sec:k-convexity} and \ref{subsec:vertex-classification} clearly demonstrate the usefulness and relevance of weakly convex sets for \textit{machine learning}. 
While our focus in this work was solely on applications to \emph{machine learning}, weakly convex sets seem to be useful for \emph{data mining} applications (e.g., itemset mining, subgroup discovery) as well. 
Another potential application area could be \textit{conceptual spaces} spanned by so-called \textit{quality dimensions}, a general framework introduced by \cite{Gaerdenfors2000,Gaerdenfors2014} for \textit{geometric} representation of concepts. 
G\"ardenfors' underlying thesis for his theory is that \textit{natural} concepts are \textit{convex} regions of conceptual spaces. 
It is an interesting question whether \textit{weak convexity} can be used effectively to \textit{decompose} concepts into semantically \textit{meaningful} ``subconcepts''. 
	
The inner loop of Algorithm~\ref{alg:intCHF} iteratively joining the blocks is very similar to \emph{single linkage clustering}, raising the following question: Can the time complexity stated in Theorem~\ref{thm:intCHF} be further improved by using techniques \citep[e.g.,][]{Sibson1973} that accelerate single linkage clustering algorithms?
	
The goal in Problems~\ref{problem:intCHF} and \ref{problem:extCHF} is to return a $\theta$-convex hull of the positive examples for the \textit{largest} $\theta$ that does not contain any of the negative examples.
This $\theta$-convex hull is, however, \textit{not} necessarily optimal with respect to the number of blocks.\footnote{
    In contrast to this long version, where the consistent hypothesis finding problems are to return a consistent weakly convex \textit{hull} of the positive examples with the \textit{smallest} number of blocks, in the short version of this paper it was \textit{mistakenly} defined to return a weakly convex \textit{set} with the \textit{smallest} number of blocks that contains all positive and none of the negative examples, and stated \textit{erroneously} that this latter problem can be solved in polynomial time for weakly convex  Boolean functions and axis-aligned hyperrectangles \citep[][Lemmas~20 and 22]{Stadtlaender2021}. 
}
The number of blocks in a $\theta$-convex set is bounded by the cardinality of the largest set $S$ satisfying $\dist(x, y) > \theta$ for all $x, y \in S$. 
For graphs, this cardinality is precisely the \emph{$\theta$-independence number}, which is NP-hard to compute \citep{Garey1979}. 
A related result of \citet{Bereg2012} states that the less restrictive problem of finding a consistent $k$-fold union of (possibly overlapping) axis-aligned hyperrectangles with \emph{minimum} $k$ is also NP-hard. 
In contrast, Problem~\ref{problem:intCHF} is computationally \textit{tractable} because the solution can be found by searching in the monotone chain of $\theta$-convex hypotheses that is uniquely defined by the training examples. 
The design and study of algorithms for the \textit{approximation} of a consistent hypothesis with the \textit{smallest} number of blocks is an interesting direction for future research.
	
The notion of weak convexity can be meaningless for certain metric spaces. 
For example, for metric spaces $(X,\dist_2)$ with \textit{finite} domains $X \subseteq \R^d$, $\ti(x, y) = \{x, y\}$ holds almost surely for all $x,y\in X$.
To overcome this problem, one can consider the following \emph{relaxation} of weak convexity which allows the triangle inequality to hold up to some \textit{tolerance} $\varepsilon$, instead of equality.
More precisely, a subset $A \subseteq X$ of a metric space $(X,\dist)$ is \emph{$(\theta,\varepsilon)$-convex} for some $\theta \geq 0$ and $\varepsilon \in [0,\theta]$, if for all $x,y \in A$ and $z \in X$ it holds that $z \in A$ whenever $\dist(x, y) \leq \theta$ and $\dist(x, z) + \dist(z, y) \leq \dist(x, y) + \epsilon$.
One can show that all results of Section~\ref{sec:weakconvexity} can naturally be generalized to this relaxed definition.
Another interesting question is whether this relaxed form of weak convexity can successfully be applied to \textit{clustering} this kind of finite point sets.

\begin{appendices}

\section*{Appendix}
\label{sec:appendix}

\begin{lemma} 
    \label{lm:VC-dimension_hyperrectangles} 
    For any constant $c\geq 0$,
    \begin{equation}
        \label{eq:vcunitdcube}
        \vcdim(\ccWConvRectangles)= \bigo(d^{c+1} \log{d})
    \end{equation}
    for all $d \in \N$ and $\theta \geq \frac{2d}{\euler} \sqrt[d]{\frac{\euler}{d^{c-1}}}$.
\end{lemma}
\begin{proof}
    Note first that all concepts in $\ccWConvRectangles<d>$ consist of a single block formed by an axis-aligned hyperrectangle in $\unitdcube$, as $D_1(x,y) \leq d$ for all $x,y \in \unitdcube$. Furthermore, they are $\theta$-convex for any $\theta > 0$.
    We claim that for $\theta$ in the lemma, 
    \begin{equation} 
        \label{eq:numberofblocks} 
	\ccWConvRectangles \subseteq \kfoldunion[\bigo\left(d^c\right)]{(\ccWConvRectangles<d>)} \enspace , \end{equation}
    i.e., each $\theta$-convex set in $\ccWConvRectangles$ is the union of at most $\bigo(d^c)$ axis-aligned
    hyperrectangles.
    Using that $\vcdim(\ccWConvRectangles<d>) = 2d$, we get (\ref{eq:vcunitdcube}) by (\ref{eq:numberofblocks})
    and \ref{thm:BEHW:k-fold-unions} of Theorem~\ref{thm:BEHW}.
 
    It remains to prove (\ref{eq:numberofblocks}). Note that for any $C \in \ccWConvRectangles$, the number of blocks in $C$ is bounded
	by the cardinality of a largest subset $S \subset \unitdcube$ satisfying $D_1(x,y) > \theta$ for all $x,y\in S$.
	We show (\ref{eq:numberofblocks}) by proving \begin{equation} 
		\label{eq:boundonk} \card{S} = \bigo(d^c) \enspace .
	\end{equation}
 	To see (\ref{eq:boundonk}), notice first that for all $v \in \unitdcube$, for the volume of the intersection of
	$\unitdcube$ with the $L_1$ $d$-ball $\dball{v} = \{u \in \R^d : D_1(v,u) \leq r\}$ we have
    \begin{equation}
		\label{eq:volume} 
		\vol\left(\dball{v} \cap \unitdcube\right) \geq \frac{\vol(\dball{v})}{2^d} \enspace .
	\end{equation}
	Thus, for all $x,y$ above it holds that $ \dball<\theta/2>{x} \cap \dball<\theta/2>{y} = \emptyset, $ which, in turn,
	implies \begin{equation} \label{eq:boundonS} \card{S} \leq  \frac{2^d
			\vol\left(\unitdcube\right)}{\vol\left(\dball<\theta/2>{v}\right)} \end{equation} by (\ref{eq:volume}).
	Using $\vol(\dball{v})=(2r)^d /d!
	$ and $d! \leq d^{d+1}/\euler^{d-1}$ \citep[see, e.g.,][]{Knuth1997}, from
	(\ref{eq:boundonS}) we have 
	\begin{equation}
		\label{eq:boundb}
		\card{S} 
		\leq \frac{d^{d+1} 2^d }{\euler^{d-1}\theta^d} = \bigo(d^c)
	\end{equation}
	for $\theta \geq \frac{2d}{\euler}
		\sqrt[d]{\frac{\euler}{d^{c-1}}}$, completing the proof of (\ref{eq:boundonk}).
\end{proof}

\end{appendices}

\bigskip
\noindent
\textbf{Acknowledgements} We are grateful to Victor Chepoi for his helpful comments on an early version of this paper and Florian Seiffarth for his technical support of the experiments.
This research has been funded by the Federal Ministry of Education and Research of Germany and the state of North Rhine-Westphalia as part of the Lamarr Institute for Machine Learning and Artificial Intelligence, LAMARR22B (\url{https://lamarr-institute.org/}). The authors gratefully acknowledge this support.

\bibliographystyle{plainnat}
\bibliography{references}

\end{document}